\documentclass[11pt]{article}
\usepackage{geometry}                
\usepackage{setspace}
\usepackage[parfill]{parskip}    
\usepackage{newtxtext}

\usepackage{amsmath}
\usepackage{amssymb}
\usepackage{amsthm}
\usepackage{bm}

\newtheorem{definition}{Definition}[section]
\newtheorem{theorem}{Theorem}[section]
\newtheorem{corollary}{Corollary}[section]
\newtheorem{lemma}{Lemma}[section]
\newtheorem{remark}{Remark}[section]

\newtheorem{assumption}{Assumption}[section]

\usepackage{tikz}
\usetikzlibrary{shapes.geometric, arrows}

\usepackage{booktabs}
\usepackage{multirow}

\usepackage[hidelinks]{hyperref}
\usepackage{url}
\usepackage{subcaption}
\usepackage{graphicx}
\usepackage{color}

\newcommand{\E}{\mathbb{E}}

\newcommand{\R}{\mathbb{R}}

\newcommand{\citep}{\cite}
\newcommand{\citet}{\cite}
\newcommand{\citealt}{\cite}

\title{Bayesian Inference Forgetting}

\author{Shaopeng~Fu\thanks{S. Fu and F. He contributed equally.}~\thanks{S. Fu, F. He, and D. Tao are with the University of Sydney. Email: 
\href{mailto:shaopengfu15@gmail.com}{shaopengfu15@gmail.com}, 
\href{mailto:fengxiang.f.he@gmail.com}{fengxiang.f.he@gmail.com}, and 
\href{mailto:dacheng.tao@gmail.com}{dacheng.tao@gmail.com}.}
 \and Fengxiang~He\footnotemark[1]~\footnotemark[2] \and Yue~Xu\thanks{Y. Xu is with University of Science and Technology of China. This work was completed when she was intern at the University of Sydney. Email: \href{mailto:xuyue502@mail.ustc.edu.cn}{xuyue502@mail.ustc.edu.cn}.}
  \and Dacheng~Tao\footnotemark[2]\\
}
        
\date{}

\begin{document}

\maketitle

\begin{abstract}
The right to be forgotten has been legislated in many countries but the enforcement in machine learning would cause unbearable costs: companies may need to delete whole models learned from massive resources due to single individual requests. Existing works propose to remove the knowledge learned from the requested data via its influence function which is no longer naturally well-defined in Bayesian inference. This paper proposes a {\it Bayesian inference forgetting} (BIF) framework to realize the right to be forgotten in Bayesian inference. In the BIF framework, we develop forgetting algorithms for variational inference and Markov chain Monte Carlo. We show that our algorithms can provably remove the influence of single datums on the learned models. Theoretical analysis demonstrates that our algorithms have guaranteed generalizability. Experiments of Gaussian mixture models on the synthetic dataset and Bayesian neural networks on the real-world data verify the feasibility of our methods. The source code package is available at \url{https://github.com/fshp971/BIF}.
\end{abstract}

\section{Introduction}

The right to be forgotten refers to the individuals' right to ask data controllers to delete the data collected from them. It has been recognized in many countries through legislation, including the European Union's General Data Protection Regulation (2016) and the California Consumer Privacy Act (2018). However, enforcement may result in unbearable costs for AI industries. When a single data deletion request comes, the data controller needs to delete the whole machine learning model which might have cost massive amounts of resources, including data, energy, and time.

To address this issue, recent works \citep{guo2019certified, golatkar2020eternal} propose to {\it forget} the influence of the requested data on the learned models via the influence function \citep{cook1982residuals,huber2004robust,koh2017understanding}: 
\begin{gather*}
\mathcal{I}(x) = -\nabla^{-2}_\theta L(\theta,S) \nabla^T_\theta \ell(\theta, x),
\end{gather*}
where $L(\theta,S)=\frac{1}{n} \sum_{i=1}^n \ell(\theta, x_i)$ is the loss function in the optimization. This approach can provably remove knowledge learned from the requested data for optimization-based machine learning subject to some conditions on convexity and smoothness of the loss function.

However, the loss function $L$ is not always naturally well-defined.
Bayesian inference approximates the posterior of probabilistic models, where a loss function may be in a different form or even do not exist.
In this case, the influence function is no longer well-defined, and therefore, existing forgetting methods become invalid.
Bayesian inference enables a wide spectrum of machine learning algorithms, such as probabilistic matrix factorization \citep{ma2014variational,salakhutdinov2008bayesian,yong2017robust}, topic models \citep{hoffman2010online,blei2003latent,zhang2020deep}, probabilistic graphic models \citep{koller2009probabilistic,beal2002infinite,ye2020optimizing} and Bayesian neural networks \citep{blundell2015weight,kingma2015variational,neal1992bayesian,titterington2004bayesian,pearce2020uncertainty,roth2018bayesian,zhang2020cyclical}.
Overlooking Bayesian inference in developing forgetting algorithms leaves a considerable proportion of machine learning algorithms costly to protect the right to be forgotten.

In this paper, we design an energy-based {\it Bayesian inference forgetting} (BIF) framework to realize the right to be forgotten in Bayesian inference. We define {\it Bayesian inference influence functions} based on the pre-defined energy functions, which help provably characterize the influence of some specific datums on the learned models. We prove that these characterizations are rigorous in first-order approximation subject to some mild assumptions on the local convexity and smoothness of the energy functions around the convergence points. Specifically, the approximation error is not larger than the order $O(1/n^2)$, where $n$ is the training sample size, which is negligible when the training sample set is sufficiently large. The BIF framework proposes to delete the Bayesian inference influence function of the requested data. We prove that BIF realizes {\it $\varepsilon$-certified knowledge removal}, a new notion defined to evaluate the performance of knowledge removal.

In the BIF framework, we develop {\it certified knowledge removal} algorithms for two canonical Bayesian inference algorithms, variational inference \citep{jordan1999introduction, blei2017variational, JMLR:v14:hoffman13a, blei2003latent} and Markov chain Monte Carlo (MCMC) \citep{geman1984stochastic,ma2015complete,welling2011bayesian,ding2014bayesian,chen2014stochastic,patterson2013stochastic}.

\begin{itemize}
\item
\textbf{Variational inference forgetting.} We show that the evidence lower bound function (ELBO; \citealt{blei2017variational}) plays the role of energy function. Based on the ELBO, a {\it variational inference influence function} is defined, which helps deliver the {\it variational inference forgetting} algorithm. We prove that the $\varepsilon$-certified knowledge removal is guaranteed on the mean-field Gaussian variational family.

\item 
\textbf{MCMC forgetting.} We formulate a new energy function for MCMC, based on which an {\it MCMC influence function} is designed. The new influence function then delivers our the {\it MCMC forgetting} algorithm. We prove that the $\varepsilon$-certified knowledge removal is achieved when the training sample set is sufficiently large.
The MCMC forgetting algorithm is also applicable in the stochastic gradient MCMC.
\end{itemize}
 
We theoretically investigated whether the BIF algorithms affect the generalizabilities of the learned Bayesian models. Based on the PAC-Bayesian theory \citep{mcallester1998some, mcallester1999pac, mcallester2003pac}, we prove that the generalization bounds are of the same order $O(\sqrt{(\log n)/n})$ ($n$ is the training sample size), no matter whether the Bayesian models are processed by BIF. Moreover, the introduced difference of the generalization bounds is not larger than the order $O(1/ n)$.

From the empirical aspect, we conduct systematic experiments to verify the forgetting algorithms.
We applied the BIF algorithms to two scenarios:
(1) Gaussian mixture models for clustering on synthetic data;
and (2) Bayesian neural networks for classification on the real-world data.
Every scenario involves variational inference, SGLD, and SGHMC.
The results suggest that the BIF algorithms effectively remove the knowledge learned from the requested data while protecting other information intact, which are in full agreement with our theoretical results.
To secure reproducibility, the source code package is available at \url{https://github.com/fshp971/BIF}.

The rest of this paper is organized as follows. Section \ref{sec:related} reviews the related works. Section \ref{sec:preliminary} provides the necessary preliminaries of Bayesian inference. 
Section \ref{sec:main_bif} presents our main result, the energy-based BIF framework.
Section \ref{sec:bif_algorithm} provides the forgetting algorithms dveloped under the BIF framework for variational inference and MCMC.
Section \ref{sec:generalization} studies the generalization abilities of the proposed algorithms.
Section \ref{sec:experiments} presents the experiment results, and Section \ref{sec:conclusion} concludes the paper.

\section{Related Works}
\label{sec:related}

\textbf{Forgetting.}
The concept ``make AI forget you" is first proposed by {Ginart \it et al.} \citet{ginart2019making}, which also designs forgetting algorithms for $k$-means. {Bourtoule \it et al.} \citet{bourtoule2019machine} then propose a ``sharded, isolated, sliced, and aggregated" (SISA) framework to approach low-cost knowledge removal
via the following three steps: (1) divide the whole training sample set into multiple disjoint shards; (2) train models in isolation on each of these shards; (3) retain the affected model when a request to unlearn a training point arrives.
{Izzo \it et al.} \citet{izzo2020approximate} develop a projective residual update (PRU) method that replaces the update vector in data removal with its projection on a low-dimensional subspace, which reduces the time complexity of data deletion to scale only linearly in the dimension of data.
{Guo \it et al.} \citet{guo2019certified} and {Golatkar \it et al.} \citet{golatkar2020eternal} propose a {\it certified} removal based on the influence function \citep{cook1982residuals,huber2004robust,koh2017understanding} of the requested data. Specifically, influence function characterizes the influence of a single datum on the learned model in the empirical risk minimization (ERM) regime. The definition of influence function relies on the gradient and Hessian of the objective function which does not naturally exist in Bayesian inference. Other advances include that {Garg \it et al.} \citet{garg2020formalizing}, {Baumhauer \it et al.} \citet{baumhauer2020machine},
and {Sommer \it et al.} \citet{sommer2020towards}.

We acknowledge that a concurrent work \citep{nguyen2020variational} also studied the knowledge removal in Bayesian models. The authors employed variational inference to minimize the KL-divergence between the approximate posterior after unlearning and the posterior of the model retrained on the remaining data. However, this approach only applies to parametric models, such as those obtained by variational inference, while many non-parametric models are un-touched, including MCMC. In contrast, our BIF covers all Bayesian inference methods. Moreover, no theoretical guarantee on the knowledge removal was established, which would be necessary to meet the legal requirements.

\textbf{Markov chain Monte Carlo (MCMC).}
{Hastings} \citet{hastings1970monte} introduces a two-step sampling method named MCMC, which first construct a Markov chain and then draw samples according to the state of the Markov chain. It is proved that the drawn samples will converge to that from the desired distribution. 
Since that, many improvements have been made on MCMC, including the Gibbs sampling method \citep{geman1984stochastic} and hybrid Monte Carlo \citep{duane1987hybrid}.
{Welling and Teh} \citet{welling2011bayesian} propose a new framework named stochastic gradient Langevin dynamics (SGLD) that by adding a proper noise to a standard stochastic gradient optimization algorithm \citep{robbins1951stochastic}, the iterates will converge to samples from the true posterior distribution as the stepsize annealed.
{Ahn \it et al.} \citet{ahn2012bayesian} extend the algorithm based on the SGLD by leveraging the Bayesian Central Limit Theorem, to improve the slow mixing rate of SGLD. Inspired by the idea of a thermostat in statistical physics, {Ding \it et al.} \citet{ding2014bayesian} leverage a small number of additional variables to stabilize momentum fluctuations caused by the unknown noise.
{Chen \it et al.} \citet{chen2014stochastic} extend the Hamiltonian Monte Carlo (HMC) to Stochastic gradient HMC (SGHMC) by adding a friction term, which enables SGHMC sample from the desired distributions without applying the MH rule.
Based on Langevin Monte Carlo methods, {Patterson and Teh} \citet{patterson2013stochastic} propose Stochastic gradient Riemannian Langevin dynamics by updating parameters according to both the stochastic gradients and additional noise which forces it to explore the full posterior. {Ma \it et al.} \citet{ma2015complete} introduce a general recipe for creating stochastic gradient MCMC samplers (SGMCMC), which is based on continuous Markov processes specified via two matrices and proved to be complete.

\textbf{Variational inference.}
{Jordan \it et al.} \citet{jordan1999introduction} introduce the use of variational inference in graphical models \citep{jordan1998learning} such as Bayesian networks and Markov random fields, in which variational inference is employed to approximate the target posterior with families of parameterized distributions.
{Blei \it et al.} \citet{blei2003latent} apply variational inference to local Dirichlet allocation (LDA), which is used to modeling the text corpora.
{Sung \it et al.} \citet{sung2008latent} propose a general variational inference framework for conjugate-exponential family models, in which the model parameters except latent variables are integrated in an exact manner, while the latent variables are approximated by a first-order algorithm.
By using stochastic optimization \citep{robbins1951stochastic}, a technique that follows noisy estimates of the gradient of the objective, {Hoffman \it et al.} \citet{JMLR:v14:hoffman13a} derive stochastic variational inference which iterates between subsampling the data and adjusting the hidden structure based only on the subsample.
{Paisley \it et al.} \citet{paisley2012variational} propose an algorithm that reduces the variance of the stochastic search gradient by using control variates based on stochastic optimization and allows for direct optimization of variational lower bound. {Titsias and L\'{a}zaro-Gredilla} \citet{aueb2015local} propose local expectation gradients, in which the stochastic gradient estimation problem over multiple variational parameters is decomposed into smaller subtasks, and each sub-task focus on the most relevant part of the variational distribution. {Zhang \it et al.} \citet{zhang2018advances} reviews recent advances of variational inference, from four aspects, scalable, generic, accurate, and amortized variational inference.

\textbf{Bayesian Neural Networks (BNNs).}
Bayesian inference is first applied to neural networks by {Neal} \citet{neal1992bayesian}, in which Hybrid Monte Carlo is employed to inference the posteriors of neural networks' parameters.
{Blundell \it et al.} \citet{blundell2015weight} propose an efficient backpropagation-compatible algorithm to calculate the gradient of the variational neural networks, which expands the applicable domain of variational inference to deep learning. Based on that, {Kingma \it et al.} \citet{kingma2015variational} further propose a local reparameterization technique to reduce the variance of stochastic gradients of variational neural networks. They also develop the variational dropout that can automatically learn the dropout rates.
{Pearce \it et al.} \citet{pearce2020uncertainty} introduce a procedures family termed randomized MAP sampling (RMS), which includes randomize-then optimize and ensemble sampling, and then realize Bayesian inference for neural networks via ensembling under the RMS family.
{Roth and Pernkopf} \citet{roth2018bayesian} successfully apply the Dirichlet process prior to BNNs, which enforces the sharing of the network weights and reduces the number of parameters.
{Zhang \it et al.} \citet{zhang2020cyclical} propose a cyclical step-size schedule for SGMCMC to learn the multimodal posterior of a Bayesian neural network.

\section{Preliminaries}
\label{sec:preliminary}


Suppose a data set is $S = \{z_1, z_2, \cdots, z_n\}$, where $z_i \in \mathcal{Z}$ is a datum and $n$ is the sample size. One may hope to fit the dataset $S$ by a parametric probabilistic model $p(z|\theta)$ where $\theta \in \Theta$ is the parameter.

\textbf{Bayesian inference} {\citep{welling2011bayesian, blei2017variational, geman1984stochastic}}
infers the posterior of the probabilistic model
\begin{gather*}
    p(\theta|S) = \frac{p(\theta) \prod_{i=1}^n p(z_i|\theta)}{p(S)},
    \label{equ:posterior}
\end{gather*}
where $p(\theta)$ is the prior of model parameter $\theta$ and 
$p(S) = \int p(\theta) \prod_{i=1}^n p(z_i|\theta) \mathrm{d}\theta$
is the normalization factor. In most cases, we do not know the closed form of the normalization factor $p(S)$. Two canonical solutions are variational inference and MCMC.

\textbf{Variational inference} {\citep{JMLR:v14:hoffman13a, blei2017variational}}
employs two steps to infer the posterior:
\begin{enumerate}
    \item Define a family of distributions, $\mathcal{Q} = \{q_\lambda | \lambda \in \Lambda\}$,
    termed variational family,
    where $\lambda$ is the variational distribution parameter.
    
    \item Search in the variational family $\mathcal{Q}$ for the distribution closest to the posterior $p(\theta|S)$ measured by KL divergence, {\it i.e.},
	$\min_\lambda \mathrm{KL}(q_\lambda(\theta) \| p(\theta|S))$.
\end{enumerate}

Minimizing the KL divergence is equivalent to maximizing the evidence lower bound (ELBO  \citealt{blei2017variational}):
\begin{equation*}
\label{equ:elbo}
\text{ELBO}(\lambda,S) = \E_{q_\lambda} \log p(\theta, S) - \E_{q_\lambda} \log q_\lambda(\theta).
\end{equation*}

A popular variational family is the mean-field Gaussian family \citep{blei2017variational,kingma2013auto,blundell2015weight,kingma2015variational}.
A mean-field Gaussian variational distribution $q_\lambda$ has the following structure,
\begin{gather*}
    q_\lambda = \mathcal{N}(\mu, \sigma^2 I),
\end{gather*}
where $\lambda = (\mu_1,\cdots,\mu_d,\sigma_1,\cdots,\sigma_d)$.

\textbf{MCMC} {\citep{ma2015complete, hastings1970monte}}
constructs a Markov chain whose stationary distribution is the targeted posterior.
The Markov chain performs a Monte Carlo procedure to sample the posterior. However, MCMC can be prohibitively time-consuming in large-scale models and large-scale data. To address the issue, stochastic gradient MCMC (SGMCMC  \citealt{ma2015complete}) introduces stochastic gradient estimation on mini-batches \citep{robbins1951stochastic} to MCMC.
In this paper, we study two major SGMCMC methods, stochastic gradient Langevin dynamics (SGLD) \citep{welling2011bayesian} and stochastic gradient Hamiltonian Monte Carlo (SGHMC) \citep{chen2014stochastic}.

SGLD updates weight $\theta_t$ as follows,
\begin{gather*}
\theta_{t+1} = \theta_t - \varepsilon_t \nabla_\theta \widetilde{U}(\theta_t) + \mathcal{N}(0,2 \varepsilon_t), \\
\widetilde{U}(\theta)=-\frac{|S|}{|\widetilde{S}|}\sum_{z\in \widetilde{S}} \log p(z|\theta) - \log p(\theta),
\end{gather*}
where $\varepsilon_t \in \R^+$ is the learning rate and $\widetilde{S}\subset S$ is the mini-batch.
Usually, one needs to anneal $\varepsilon_t$ to small values.

SGHMC introduces a momentum $v_t$ to the weight update:
\begin{align*}
&\theta_{t+1} = \theta_{t} + v_{t}, \\
&v_{t+1} = (1-\alpha_t)v_t -\eta_t \nabla\widetilde{U}(\theta_t) + \mathcal{N}(0, 2\alpha_t\eta_t),
\end{align*}
where $\alpha_t \propto \sqrt{\eta_t}$ (see eqs. (14), (15) in \citet{chen2014stochastic}),
and the momentum is initialized as $v_0 \sim \mathcal{N}(0,\eta_0)$.
Moreover, when $\eta_t$ is decay by a factor of $c$, $\alpha_t$ needs to be decayed by a factor of $\sqrt{c}$.

\section{Bayesian Inference Forgetting}
\label{sec:main_bif}
This section presents our Bayesian inference forgetting framework.
We first define $\varepsilon$-certified knowledge removal in Bayesian inference, which quantitates the knowledge removal performance. Then, we study the energy functions in Bayesian inference.
The forgetting algorithm for Bayesian inference is designed based on the energy functions.

\subsection{$\varepsilon$-Certified Knowledge Removal}
\label{sec:prob_def}

Suppose a Bayesian inference method learns a probabilistic model $\hat p_S$ on the sample set $S$. A client requests to remove her/his data $S' \subset S$. A trivial and costly approach is re-training the model on data $S - S'$. Suppose the re-learned probabilistic model is $\hat p_{S - S'}$.

A forgetting algorithm $\mathcal{A}$ is designed to process the distribution $\hat p_{S}$ as follows to estimate the distribution $\hat p_{S- S'}$,
\begin{gather*}
    \hat p^{-S'}_{S} = \mathcal{A}(\hat p_{S},S').
\end{gather*}
We term the distribution $\hat p^{-S'}_{S}$ as the {\it processed model}.

In order to meet the regulation requirements, the algorithm needs to achieve {\it $\varepsilon$-certified knowledge removal} as follows.


\begin{definition}[$\varepsilon$-certified knowledge removal in Bayesian inference]
For any subset $S' \subset S$ and $\varepsilon > 0$, we call $\mathcal{A}$ performs $\varepsilon$-certified knowledge removal, if
\begin{gather*}
    \mathrm{KL}(\hat p^{-S'}_S,\hat p_{S-S^\prime}) \leq \varepsilon.
\end{gather*}
\end{definition}

\subsection{Energy Functions in Bayesian Inference}

Bayesian inference can be formulated as minimizing an energy function, echos the free-energy principle in physics. This section studies the energy functions in Bayesian inference.

Suppose $F(\gamma,S)$ is the energy function of a probabilistic distribution $\chi(\gamma)$ parametrized by $\gamma \in \Gamma \subset \R^K$ over the model parameter space $\Theta$.
A typical energy function $F(\gamma,S)$ has the following structure,
\begin{gather}
    F(\gamma, S) = \sum_{i=1}^n h(\gamma,z_i) + f(\gamma),
    \label{equ:prob_obj_S}
\end{gather}
where $h(\gamma,z)$ is a function defined on $\Gamma\times\mathcal{Z}$ that characterizes the influence from individual datums, 
$f(\gamma)$ is a function defined on $\Gamma$ that characterizes the influence from the prior of $\gamma$. Usually, a lower energy corresponds to a smaller distance between the two distributions.
In variational inference, ELBO plays as an energy function. In MCMC, we construct an energy function.
Please see Sections \ref{sec:variational_if} and \ref{sec:mcmc_if} for more details.

Similarly, the energy function for $S-S^\prime$ is $F(\gamma,S-S^\prime)$.
It can be re-arranged as follows,
\begin{gather}
    F(\gamma, S-S^\prime)
    = F(\gamma,S) - \sum_{z\in S^\prime} h(\gamma,z). \label{equ:prob_obj_rm_set}
\end{gather}

\subsection{Forgetting Algorithm for Bayesian inference}

Based on the energy function, the forgetting algorithm for Bayesian inference is then designed and shown as fig. \ref{fig:bif_flowchart}.

\begin{figure}[h]
\centering
\tikzstyle{arrow} = [->,>=stealth]
\tikzstyle{mycircle} = [circle, minimum width=1.2cm]
\newcommand{\mydist}{1.6cm}
\begin{tikzpicture}[node distance=\mydist]

    \node[draw, rectangle,
          minimum width=4.8cm,
          minimum height=3.4cm,
          xshift=3.4cm,
          yshift=-0.65cm
        ] (name) {};

    \node[draw, mycircle] (z_j) {$z_j$};
    \node[coordinate, right of=z_j, xshift=0.2cm] (point1) {};
    \node[draw, mycircle, right of=point1] (derivative) {$\frac{\partial \hat\gamma^{-z_j}_S(0)}{\partial \tau}$};
    
    \node[draw, mycircle, above of=point1, yshift=0.2cm] (S) {$S$};

    \node[draw, mycircle, below of=z_j] (p_0) {$\hat p_S$};
    \node[draw, mycircle, right of=p_0, xshift=0.2cm] (gamma_0) {$\gamma_S$};
    \node[coordinate, right of=gamma_0] (point2) {};

    \node[draw, mycircle, right of=point2] (gamma_1) {$\gamma^{-z_j}_S$};
    \node[draw, mycircle, right of=gamma_1, xshift=0.2cm] (p_1) {$\hat p^{-z_j}_S$};

    \draw [arrow] (z_j) -- (derivative);
    \draw [arrow] (p_0) -- (gamma_0);
    \draw [arrow] (gamma_0) -- (derivative);
    \draw [arrow] (S) -- (derivative);
    \draw [arrow] (gamma_0) -- (gamma_1);
    \draw [arrow] (derivative) -- (gamma_1);
    \draw [arrow] (gamma_1) -- (p_1);
\end{tikzpicture}
\caption{The workflow of Bayesian inference forgetting framework that removes the influence of datum $z_j \in S$ from the distribution $\hat p_S$.
$\hat p_S$ is the distribution learned on the training sample set $S$ and is parameterized by $\gamma_S$.
$\hat p^{-z_j}_S$ is the processed distribution and is parameterized by $\gamma^{-z_j}_S$.
$\gamma^{-z_j}_S = \gamma_S - \frac{\partial \hat\gamma^{-z_j}_S(0)}{\partial \tau}$.}
\label{fig:bif_flowchart}
\end{figure}
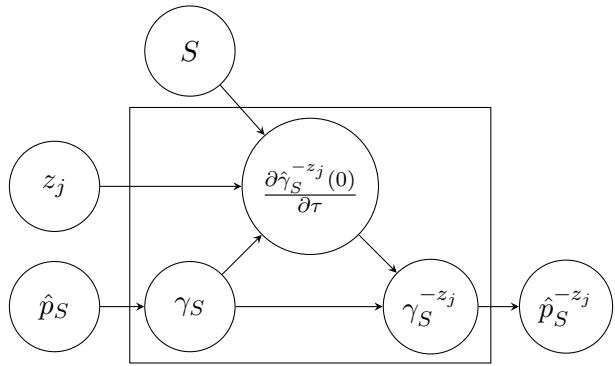

Suppose that the probabilistic model $\hat p_S$ learned on the training sample set $S$ is parameterized by $\gamma_S$.
Then, $\gamma_S$ is a local minimizer of the energy function $F(\gamma,S)$.
Similarly, suppose that the model $\hat p_{S-S'}$ learned on $S-S'$ is parameterized by $\gamma_{S-S'}$.
Then $\gamma_{S-S'}$ is a local minimizer of the energy function $F(\gamma,S-S')$.
Also, we assume the processed model $\hat p^{-S'}_S$ is parameterized by $\gamma^{-S'}_S$.

Mathematically, the forgetting in Bayesian inference can be realized by approaching the local minimizer $\gamma_S$ of the energy function $F(\gamma,S)$ to the local minimizer $\gamma_{S-S^\prime}$ of the energy function $F(\gamma,S - S')$.

We start with 
a simple case that remove the influence learned from a single datum $z_j \in S$. The energy function for the posterior $p(\theta|S-\{z_j\})$ is as follows,
\begin{align}
    F(\gamma, S-\{z_j\})
    &= \sum_{i=1}^n h(\gamma,z_i) + f(\gamma) - h(\gamma,z_j).
    \label{equ:prob_obj_rm_z0}
\end{align}
Let $\gamma_{S-\{z_j\}}$ be a local minimizer of $F(\gamma,S-\{z_j\})$. Then, 
\begin{gather*}
    \nabla_{\gamma}F(\gamma_S, S) = \sum_{i=1}^n\nabla_{\gamma}h(\gamma_S,z_i) + \nabla_{\gamma}f(\gamma_S) = 0, \\
    \nabla_{\gamma}F(\gamma_{S-\{z_j\}},S-\{z_j\}) = \nabla_{\gamma}F(\gamma_{S-\{z_j\}},S) - \nabla_{\gamma}h(\gamma_{S-\{z_j\}},z_j) = 0.
\end{gather*}
Notice that the structures of the above two equations only differ in the term $\nabla_{\gamma}h(\gamma,z_j)$. They are two special cases of the following equation,
\begin{gather}
    \nabla_{\gamma}F(\gamma,S) + \tau \cdot \nabla_{\gamma}h(\gamma,z_j) = 0,
    \label{equ:prob_diff}
\end{gather}
where $\tau \in [-1,0]$.
Eq. (\ref{equ:prob_diff}) induces an implicit mapping $\hat\gamma^{-z_j}_S:[-1,0]\to\Gamma$ such that $\hat\gamma^{-z_j}_S(0)=\gamma_S$.
For any $\tau \in [-1,0]$, $\hat\gamma^{-z_j}_S(\tau)$ is a solution of eq. (\ref{equ:prob_diff}).
Thus, $\hat\gamma^{-z_j}_S(\tau)$ is also a critical point of the following function,
\begin{gather}
    F_{-z_j,\tau}(\gamma,S) = F(\gamma,S) + \tau \cdot h(\gamma,z_j).
    \label{equ:prob_obj_eps}
\end{gather}
When $\tau=-1$, $F_{-z_j,-1}(\gamma,S)$ is exactly $F(\gamma,S-\{z_j\})$, which indicates that $\hat\gamma^{-z_j}_S(-1)$ is a critical point of $F(\gamma,S-\{z_j\})$, and is possible to be a local minimizer. We assume that $\hat\gamma^{-z_j}_S(-1)$ is really a local minimizer of $F(\gamma,S-\{z_j\})$.
We also let $\gamma_{S-\{z_j\}} = \hat\gamma^{-z_j}_S(-1)$. Then, $\gamma_{S-\{z_j\}}$ can be approached based on $\gamma_S$ in a first-order approximation manner,
\begin{gather*}
    \gamma_{S-\{z_j\}} = \hat\gamma^{-z_j}_S(-1)
    \approx \hat\gamma^{-z_j}_S(0) - \frac{\partial\hat\gamma^{-z_j}_S(0)}{\partial\tau}
    = \gamma_S - \frac{\partial\hat\gamma^{-z_j}_S(0)}{\partial\tau}.
\end{gather*}
Thus, when a request of removing a datum $z_j$ comes, our forgetting algorithm will process the request by replacing $\gamma_S$ with $\gamma^{-z_j}_S$, in which
\begin{gather*}
\gamma^{-z_j}_S := \gamma_S - \frac{\partial\hat\gamma^{-z_j}_S(0)}{\partial\tau}.
\end{gather*}

\subsection{Theoretical Guarantee}
Under some mild assumptions, we prove that {(1) the mapping $\hat\gamma^{-z_j}_S$ uniquely exists;} (2) $\gamma_{S-\{z_j\}}=\hat\gamma_S(-1)$ is the global minimizer of $F(\gamma,S-\{z_j\})$; and (3) the approximation error between $\gamma^{-z_j}_S$ and $\gamma_{S-\{z_j\}}$ is not larger than order $O\left(1/n^2\right)$, where $n$ is the training set size.

We first introduce two definitions.

\begin{definition}[compact set]
A set is called compact if and only if it is closed and bounded.
\end{definition}

\begin{definition}[strong convexity]
\label{def:strong_convexity}
A differentiable function $f$ is called $c$-strongly convex, if and only if there exists a constant real $c>0$, such that for any two points $z$ and $z^\prime$ in the domain of the function $f$, the following inequality holds,
\begin{gather*}
(\nabla f(z) - \nabla f(z^\prime))^T (z-z^\prime) \geq \frac{c}{2} \|z-z^\prime\|^2_2.
\end{gather*}
When $f$ is second-order continuously differentiable, for any $c>0$, the following three propositions are equivalent:
\begin{enumerate}
    \item $f$ is $c$-strongly convex;
    \item for any $z$ from the domain of the function $f$, the smallest eigenvalue of the Hessian matrix $\nabla^2 f(z)$ is at least $c$, {\it i.e.}, $\lambda_{\min} \left( \nabla^2 f(z) \right) \geq c$;
    \item for any $z$ from the domain of $f$, the matrix $(\nabla^2 f(z) - cI)$ is positive semi-definite, {\it i.e.}, $\nabla^2 f(z) \succeq cI$.
\end{enumerate}
\end{definition}

The theoretical guarantees rely on the following two assumptions.

\begin{assumption}
Both of the supports $\Gamma$ and $\mathcal{Z}$ are compact.
\label{ass:compact}
\end{assumption}

Usually, the energy functions $F(\gamma,S)$ and $F(\gamma,S-\{z_j\})$ are close to each other.
Thus, it is reasonable to assume that the two local minimizers $\gamma$ and $\gamma_{S-\{z_j\}}$ fall in the same local region.
Hence, we only need to analyze on a limited local region, which justifies Assumption \ref{ass:compact}.

\begin{assumption}
\label{ass:cont_diff_scox}
Suppose $f(\gamma)$ and $h(\gamma,z)$ are the two inputs of the energy functions $F(\gamma,S)$. We assume that
\begin{enumerate}
    \item $f(\gamma)$ is $3$rd-order continuously differentiable and $c_f$-strongly convex on $\Gamma$,
    \item $h(\gamma,z)$ is $3$rd-order continuously differentiable with respect to $\gamma$ in $\Gamma\times\mathcal{Z}$. Besides, $\forall z\in\mathcal{Z}$, $h(\gamma,z)$ is $c_h(z)$-strongly convex with respect to $\gamma$, where $c_h:\mathcal{Z}\to\R$ is a continuous function.
\end{enumerate}
\end{assumption}

Assumption \ref{ass:cont_diff_scox} is mild in gradient-based optimization.
Together with Assumption \ref{ass:compact}, they assume the local strong convexity and smoothness of the energy functions.

We then prove that the energy function $F_{-z_j,\tau}$ is strongly convex, as the following lemma.

\begin{lemma}
\label{lem:scox_eps}
Suppose Assumption \ref{ass:cont_diff_scox} holds. For any $\tau \in [-1,0]$, $F_{-z_j,\tau}(\gamma,S)$ (eq. (\ref{equ:prob_obj_eps})) is strongly convex with respect to {the parameter} $\gamma$.
\end{lemma}
\begin{proof}
We rewrite the energy function $F_{-z_j,\tau}$ as follows,
\begin{gather*}
F_{-z_j,\tau}(\gamma,S)=F(\gamma,S-\{z_j\}) + (1+\tau)h(\gamma,z_j).
\end{gather*}
Apparently, $F(\gamma,S-\{z_j\})$ is strongly convex. Since $\tau\in[-1,0]$, we have that $1+\tau\geq 0$. Thus, $(1+\tau)h(\gamma,z_j)$ is either strongly convex or equal to zero. Hence, we have that $F_{-z_j,\tau}(\gamma,S)$ is strongly convex with respect to $\gamma$.

The proof is completed.
\end{proof}

We are now able to prove that eq. (\ref{equ:prob_diff}) really induces a continuous mapping $\hat\gamma^{-z_j}_S$, and $\hat\gamma^{-z_j}_S(-1)$ is the global minimizer of $F(\gamma,S-\{z_j\})$.

\begin{theorem}
\label{thm:imp_fun}
Suppose assumption \ref{ass:cont_diff_scox} holds. Then, eq. (\ref{equ:prob_diff}) induces a unique continuous function $\hat\gamma^{-z_j}_S:[-1,0]\to\Gamma$ such that for any $\tau\in[-1,0]$, $\hat\gamma^{-z_j}_S(\tau)$ is the solution of eq. (\ref{equ:prob_diff}) and the global minimizer of eq. (\ref{equ:prob_obj_eps}) with respect to $\gamma$.
\end{theorem}

\begin{proof}
According to Lemma \ref{lem:scox_eps}, for any $\tau \in [-1,0]$, the energy function $F_{-z_j, \tau}(\gamma,S)$ is strongly convex with respect to $\gamma$.
Thus, the global minimizer of the energy function $F_{-z_j,\tau}(\gamma,S)$ is unique with respect to $\gamma$.
Therefore, we can define the mapping $\hat\gamma^{-z_j}_S$ as follows,
\begin{gather*}
    \hat\gamma^{-z_j}_S(\tau) = \arg\min_{\gamma} F_{-z_j,\tau}(\gamma,S),
\end{gather*}
where $\tau \in [-1,0]$.

Besides, The strong convexity of $F_{-z_j, \tau}(\gamma,S)$ indicates that when $\tau$ is fixed, a parameter $\gamma^* \in \Gamma$ satisfies $\nabla_\gamma F_{-z_j,\tau}(\gamma^*,S)=0$ (eq. (\ref{equ:prob_diff})), if and only if $\gamma^*$ is the global minimizer of $F_{-z_j,\tau}(\gamma,S)$.
Thus, $\forall \tau \in [-1,0]$, $\hat\gamma^{-z_j}_S(\tau)$ is the only solution of eq. (\ref{equ:prob_diff}).

Eventually, according to Definition \ref{def:strong_convexity}, the strong convexity of $F_{-z_j,\tau}(\gamma,S)$ implies that the Hessian matrix $\nabla^2_\gamma F_{-z_j,\tau}(\hat\gamma^{-z_j}_S(\tau),S)$ is invertible. Combining with the implicit function theorem, we have that $\hat\gamma^{-z_j}_S(\tau)$ is continuous everywhere on $[-1,0]$.

The proof is completed.
\end{proof}

Theorem \ref{thm:imp_fun} demonstrates that our algorithm can obtain the $\hat\gamma^{-z_j}_S(-1)$. This secures that our algorithm can remove the learned influence from the probabilistic model $\hat p_S$.

We then analyze the error introduced by our algorithm. Assume that $\hat\gamma^{-z_j}_S$ is second-order continuously differentiable, then $\hat\gamma^{-z_j}_S(-1)$ can be expanded by the {Taylor's series}:
\begin{align*}
    \hat\gamma^{-z_j}_S(-1) \nonumber
    &= \hat\gamma^{-z_j}_S(0) + (-1) \cdot \frac{\partial\hat\gamma^{-z_j}_S(0)}{\partial\tau} + \frac{(-1)^2}{2} \cdot \frac{\partial^2\hat\gamma^{-z_j}_S(\xi)}{\partial\tau^2} \nonumber \\
    &= \hat\gamma^{-z_j}_S(0) - \frac{\partial\hat\gamma^{-z_j}_S(0)}{\partial\tau} + \frac{1}{2} \cdot \frac{\partial^2\hat\gamma^{-z_j}_S(\xi)}{\partial\tau^2},
\end{align*}
where $\xi\in[-1,0]$, $\frac{1}{2} \frac{\partial^2\hat\gamma^{-z_j}_S(\xi)}{\partial\tau^2}$ is the Cauchy form of the remainder. We prove that the Cauchy remainder term becomes negligible as the training set size $n$ goes to infinity.

\begin{theorem}
\label{thm:1st_norm}
Suppose Assumptions \ref{ass:compact} and \ref{ass:cont_diff_scox} hold. The induced mapping $\hat\gamma^{-z_j}_S$ is as defined in Theorem \ref{thm:imp_fun}. Then, for any $\tau \in [-1,0]$, we have that
\begin{gather*}
    \frac{\partial\hat\gamma^{-z_j}_S(\tau)}{\partial\tau}
    =-\left(\nabla^2_\gamma F(\hat\gamma^{-z_j}_S(\tau),S) + \tau\cdot\nabla^2_\gamma h(\hat\gamma^{-z_j}_S(\tau),z_j)\right)^{-1}
     \cdot \nabla_\gamma h(\hat\gamma^{-z_j}_S(\tau),z_j)^T,
\end{gather*}
and
\begin{gather*}
    \left\|\frac{\partial\hat\gamma^{-z_j}_S(\tau)}{\partial\tau}\right\|_2
    \leq O\left(\frac{1}{n}\right).
\end{gather*}
\end{theorem}

The proof of Theorem \ref{thm:1st_norm} is presented in Appendix \ref{app:proof_bif_framework}.
We first calculate
$\frac{\partial\hat\gamma^{-z_j}_S(\tau)}{\partial\tau}$
based on eq. (\ref{equ:prob_diff}).
We then upper bound the norm of
$\frac{\partial\hat\gamma^{-z_j}_S(\tau)}{\partial\tau}$
by upper bounding the norm of an inverse matrix
\begin{gather*}
\left(\nabla^2_\gamma F(\hat\gamma^{-z_j}_S(\tau),S) + \tau\cdot\nabla^2_\gamma h(\hat\gamma^{-z_j}_S(\tau),z_j)\right)^{-1}
\end{gather*}
and a vector
$\nabla_\gamma h(\hat\gamma^{-z_j}_S(\tau),z_j)$.
The norm of the inverse matrix is bounded based on the strong convexity of $\hat\gamma^{-z_j}_S(\tau)$,
while the norm of the vector is bounded based on its continuity in a compact domain.

Based on Theorem \ref{thm:1st_norm}, we then prove that the norm of $\frac{\partial^2\gamma^{-z_j}_S(\tau)}{\partial\tau^2}$ is no larger than order $O(1/n^2)$.
Thus, a sufficiently large training sample size ensures that, the second and high order terms in the Taylor's series are negligible.

\begin{theorem}
\label{thm:2nd_norm}
Suppose Assumptions \ref{ass:compact} and \ref{ass:cont_diff_scox} hold. The induced mapping $\hat\gamma^{-z_j}_S$ is as defined in Theorem \ref{thm:imp_fun}. Then, for any $\tau\in[-1,0]$, we have that
\begin{gather*}
    \left\|\frac{\partial^2\hat\gamma^{-z_j}_S(\tau)}{\partial\tau^2}\right\|_2 \leq O\left(\frac{1}{n^2}\right).
\end{gather*}
\end{theorem}

The proof routine of Theorem \ref{thm:2nd_norm} is similar to Theorem \ref{thm:1st_norm}, while
the calculation is more difficult.
We leave the details in Appendix \ref{app:proof_bif_framework}.

\begin{corollary}
\label{cor:bif_rm_z}
Suppose Assumptions \ref{ass:compact} and \ref{ass:cont_diff_scox} hold. Suppose $\gamma_S$ and $\gamma_{S-\{z_j\}}$ are the global minimizers of the energy functions $F(\gamma,S)$ and $F(\gamma,S-\{z_j\})$, respectively.
Then, we have that
\begin{gather*}
\gamma_{S-\{z_j\}} = \gamma_S + \nabla^{-2}_\gamma F(\gamma_S,S) \cdot \nabla_\gamma h(\gamma_S,z_j)^T + O\left(\frac{1}{n^2}\right).
\end{gather*}
\end{corollary}

Eventually, we consider removing a subset $S^\prime$ from $S$. Follow the previous derivations, we prove that the approximation error of approaching $\gamma_{S-S^\prime}$ based on $\gamma_S$ is also not larger than order $O(1/n^2)$, as stated below:

\begin{corollary}
\label{cor:bif_rm_S_prime}
Suppose Assumptions \ref{ass:compact} and \ref{ass:cont_diff_scox} hold. Suppose $\gamma_S$ and $\gamma_{S-S^\prime}$ are the global minimizers of the energy functions $F(\gamma,S)$ and $F(\gamma,S- S^\prime)$, respectively.
Then, we have that
\begin{gather*}
\gamma_{S-S^\prime} = \gamma_S + \nabla^{-2}_\gamma F(\gamma_S,S) \cdot \sum_{z_j \in S^\prime} \nabla_\gamma h(\gamma_S,z_j)^T +O\left(\frac{1}{n^2}\right).
\end{gather*}
\end{corollary}

\section{Forgetting Algorithms for Bayesian Inference}
\label{sec:bif_algorithm}
In this section, we develop provably certified forgetting algorithms for variational inference and MCMC under the BIF framework.

\subsection{Variational Inference Forgetting}
\label{sec:variational_if}

As introduced in Section \ref{sec:preliminary}, variational inference aims to minimize the negative ELBO function, which can be expanded as follows,
\begin{align*}
&-\text{ELBO}(\lambda,S) \nonumber \\
&\quad = -\E_{q_\lambda} \log p(\theta, S) + \E_{q_\lambda} \log q_\lambda(\theta) \nonumber \\
&\quad = -\E_{q_\lambda} \log \left( p(\theta) \prod_{i=1}^n p(z_i|\theta) \right) + \E_{q_\lambda} \log q_\lambda(\theta) \nonumber \\
&\quad = \sum_{i=1}^n -\E_{q_\lambda}\log p(z_i|\theta) + \mathrm{KL}(q_\lambda(\theta)\|p(\theta)).
\end{align*}
Notice that the structure of the above equation is similar as that of the energy function (eq. (\ref{equ:prob_obj_S})). Thus, we employ the negative ELBO function as the energy function.

Based on the energy function, we define the following variational inference influence function to characterize the influence of single datum on the learned model.

\begin{definition}[Variational inference influence function]
For any example $z_j \in S$, its variational inference influence function is defined to be
\begin{gather*}
    \mathcal{I}_{\mathrm{VI}} (z_j) :=
    - \nabla^{-2}_\lambda \mathrm{ELBO}(\hat\lambda_S,S) \cdot \nabla_{\lambda}^T \E_{q_{\hat\lambda_S}} \log p(z_j|\theta).
\end{gather*}
\end{definition}

We then design the {\it variational inference forgetting} algorithm, to remove the influences of a single datum from the learned variational parameter $\hat\lambda_S$,
\begin{gather*}
    \hat\lambda^{-z_j}_S = \mathcal{A}_{\mathrm{VI}}(\hat\lambda_S, z_j) = \hat\lambda_S - \mathcal{I}_{\mathrm{VI}}(z_j).
\end{gather*}

We prove that variational inference forgetting can provably remove the influences of datums from the learned model.
\begin{theorem}
\label{thm:vi_bif}
Let $\gamma:=\lambda$, $\Gamma:=\Lambda$, $h(\gamma,z):=-\E_{q_\lambda}\log p(z|\theta)$ and $f(\gamma):=\mathrm{KL}(q_\lambda(\theta)\|p(\theta))$. Suppose Assumptions \ref{ass:compact} and \ref{ass:cont_diff_scox} hold. Then, we have that
\begin{gather*}
    \hat\lambda_{S-\{z_j\}}
    = \hat\lambda_S - \mathcal{I}_{\mathrm{VI}}(z_j) + O\left(\frac{1}{n^2}\right)
    \approx \hat\lambda_S - \mathcal{I}_{\mathrm{VI}}(z_j),
    \label{equ:variational_forget}
\end{gather*}
where $\hat\lambda_S$ and $\hat\lambda_{S-\{z_j\}}$ are the global minimizers of $-\mathrm{ELBO}(\lambda,S)$ and $-\mathrm{ELBO}(\lambda,S-\{z_j\})$, respectively.
\end{theorem}

\begin{proof}
It is straightforward from Corollary \ref{cor:bif_rm_z}.
\end{proof}

\begin{remark}
When Assumptions \ref{ass:compact} and \ref{ass:cont_diff_scox} do not hold, performing variational inference forgetting approaches some critical point of $-\mathrm{ELBO}(\lambda,S-\{z_j\})$ based on $\hat\lambda_S$.
\end{remark}

Beyond a single datum's removal, Corollary \ref{cor:bif_rm_S_prime} further shows that the influence of a sample set $S^\prime$ can also be characterized by the variational inference influence function $\mathcal{I_{\mathrm{VI}}}(S^\prime)$, as follows,
\begin{gather*}
    \mathcal{I}_{\mathrm{VI}}(S^\prime) = \sum_{z_j\in S^\prime} \mathcal{I}_{\mathrm{VI}}(z_j).
\end{gather*}
This guarantees that one can remove a group of datums at one time. It can significantly speed up and simplify {the forgetting process} of large amounts of datums.

We next study the $\varepsilon$-certified knowledge removal guarantee for variational inference forgetting.
Here, we take the mean-field Gaussian family as an example.
Proofs for other varitional families are similar.

\begin{theorem}[$\varepsilon$-certified knowledge removal of mean-field Gaussian variational distribution]
\label{thm:vi_certified}
Let $\mathcal{Q}$ be a mean-field Gaussian family where the variance of every variational distributions are bounded,
{\it i.e.}, $\exists M_1, M_2 \in \R$ such that for any $q_\lambda = \mathcal{N}(\mu,\sigma^2 I) \in \mathcal{Q}$, we have $0 < M_1 \leq \sigma_i \leq M_2$ holds for $i = 1, \cdots, d$.
Let $\hat\lambda_S$ and $\hat\lambda_{S-\{z_j\}}$ be the variational parameters that learned on sample sets $S$ and $S-\{z_j\}$, respectively. Suppose $\hat\lambda^{-z_j}_S = \mathcal{A}_{\mathrm{VI}}(\hat\lambda_S,z_j)$ is the processed variational parameter.
Then, $\mathcal{A}_{\mathrm{VI}}(\hat\lambda_S,z_j)$ performs $\varepsilon_{\hat\lambda_S,z_j}$-certified knowledge removal, where
\begin{gather*}
\varepsilon_{\hat\lambda_S,z_j} = \frac{1}{2 M_1^2}\left(2(M_1+M_2)\|\hat\lambda^{-z_j}_S - \hat\lambda_{S-\{z_j\}}\|_1 + \|\hat\lambda^{-z_j}_S - \hat\lambda_{S-\{z_j\}}\|_2^2 \right).
\end{gather*}
\end{theorem}

The proof of Theorem \ref{thm:vi_certified} is omitted here and presented in Appendix \ref{app:proof_bif_certified}.
When all the conditions in Theorem \ref{thm:vi_bif} hold, we have that
$\|\hat\lambda^{-z_j} - \hat\lambda_{S-\{z_j\}}\|_1\leq O(1/n^2)$
and
$\|\hat\lambda^{-z_j}_S - \hat\lambda_{S-\{z_j\}}\|^2_2\leq O(1/n^4)$.
Thus, $\varepsilon_{\hat\lambda_S,z_j} \leq O(1/n^2)$.
Therefore, $\mathcal{A}_{\mathrm{VI}}$ performs $O(1/n^2)$-certified knowledge removal for mean-field Gaussian variational distribution $q_{\hat\lambda_S}$.

\subsection{MCMC Forgetting Algorithm}
\label{sec:mcmc_if}

In MCMC, there is no explicit objective function.
In this work, we transport $p(\theta|S)$ to approximate $p(\theta|S^\prime)$. Specifically, we want to find a ``drift'' $\hat\Delta^{-S^\prime}_S$ named {\it drifting influence} to minimize the following KL divergence,
\begin{gather}
    \hat\Delta^{-S^\prime}_S = \arg\min_{\Delta}\mathrm{KL}(p(\theta|S) \| p(\theta+\Delta|S- S^\prime)).
    \label{equ:mcmc_formulate}
\end{gather}

The KL term in eq. (\ref{equ:mcmc_formulate}) can be expanded as follows,
\begin{align}
&\mathrm{KL}(p(\theta|S) \| p(\theta+\Delta|S- S^\prime)) \nonumber \\
&\quad = \E_{p(\theta|S)}\log p(\theta|S) - \E_{p(\theta|S)}\log p(\theta+\Delta|S- S^\prime) \nonumber \\
&\quad = -\E_{p(\theta|S)}\log\left( \frac{p(\theta+\Delta) \prod_{z\in S} p(z|\theta+\Delta)}{p(S- S^\prime) \prod_{z \in S^\prime} p(z|\theta+\Delta)} \right) + \E_{p(\theta|S)}\log p(\theta|S) \nonumber \\
&\quad = \sum_{i=1}^n -\E_{p(\theta|S)}\log p(z_i|\theta+\Delta) -\E_{p(\theta|S)}\log p(\theta+\Delta) \nonumber \\
&\quad \quad - \sum_{z\in S^\prime} -\E_{p(\theta|S)}\log p(z|\theta+\Delta) + \mathrm{Constant}.
\end{align}
Despite the constant term, the remaining has the similar structure with $F(\gamma,S- S^\prime)$ (eq. (\ref{equ:prob_obj_rm_set})).
Thus, we define the energy function for MCMC by replacing $h(\gamma,z)$ and $f(\gamma)$ with $-\E_{p(\theta|S)}\log p(z|\theta+\Delta)$ and $-\E_{p(\theta|S)}\log p(\theta+\Delta)$, respectively.

Based on the energy function, we define an MCMC influence function to characterize the {drifting influences} induced by a single datum as follows.
\begin{definition}[MCMC influence function]
\label{def:mcmc_if}
For any example $z_j \in S$, its MCMC influence function is defined to be
\begin{align}
    &\mathcal{I}_{\mathrm{MCMC}}(z_j)
    := - \left(\E_{p(\theta|S)}\nabla_\theta^2\log p(\theta,S)\right)^{-1} \left(\E_{p(\theta|S)} \nabla_\theta \log p(z_j|\theta)\right)^T.
    \label{equ:mcmc_if}
\end{align}
\end{definition}

Then, the {\it MCMC forgetting} algorithm is as follows,
\begin{gather*}
    p^{-z_j}_S(\theta) = \mathcal{A}_{\mathrm{MCMC}}(p(\theta|S),z_j) = p(\theta+\mathcal{I}_{\mathrm{MCMC}}(z_j)|S).
\end{gather*}

In practice, we do not have the posterior $p(\theta|S)$, but only some samples drawn from $p(\theta|S)$.
Thus, we are not able to transform the posterior $p(\theta|S)$ to the processed distribution $p^{-z_j}_S(\theta)$.
However, the probability of drawing a sample $\theta_t$ from $p(\theta|S)$ equals that of drawing a sample $\theta_t-\mathcal{I}_{\mathrm{MCMC}}(z_j)$ from $p^{-z_j}_S(\theta)$.
Therefore, when performing MCMC forgetting,
we will replace any drawn sample $\theta_t$ by a new sample $\theta_t-\mathcal{I}_{\mathrm{MCMC}}(z_j)$.


We prove that the MCMC forgetting algorithm can provably remove the {drifting influence} induced by a single datum from the learned model.
\begin{theorem}
\label{thm:mcmc_if}
Let $\gamma:=\Delta$, $\Gamma:=\mathrm{supp}(\Delta)$, $h(\gamma,z):= -\E_{p(\theta|S)}\log p(z|\theta+\Delta)$ and $f(\gamma):= -\E_{p(\theta|S)}\log p(\theta+\Delta)$. Suppose that Assumptions \ref{ass:compact} and \ref{ass:cont_diff_scox} hold. 
Then, we have that
\begin{align*}
    \hat\Delta^{-z_j}_S
    &= -\mathcal{I}_{\mathrm{MCMC}}(z_j) + O\left(\frac{1}{n^2}\right)
    \approx -\mathcal{I}_{\mathrm{MCMC}}(z_j),
\end{align*}
where $\hat\Delta^{-z_j}_S$ is the global minimizer of $\mathrm{KL}(p(\theta|S) \| p(\theta+\Delta|S-\{z_j\}))$.
\end{theorem}

\begin{remark}
When Assumptions \ref{ass:compact} and \ref{ass:cont_diff_scox} do not hold, performing MCMC forgetting can be seen as finding a critical point $\hat\Delta^{-z_j}_S$ of $\mathrm{KL}(p(\theta-\Delta|S) \| p(\theta|S-\{z_j\}))$.
\end{remark}

\begin{proof}
When all the conditions hold, the global minimizer of $\mathrm{KL}(p(\theta|S)\|p(\theta+\Delta|S))$ is exactly $0$.
Meanwhile, $\hat\Delta^{-z_j}_S$ is the global minimizer of $\mathrm{KL}(p(\theta|S) \| p(\theta+\Delta|S-\{z_j\}))$.
Combining Corollary \ref{cor:bif_rm_z}, we have that
\begin{align}
\hat\Delta^{-z_j}_S
=& \hat\Delta^{-z_j}_S - 0 \nonumber \\
=&\left(\nabla^2_\Delta \left[ -\E_{p(\theta|S)}\log p(\theta+\Delta,S) \right]_{\Delta=0} \right)^{-1}  \nonumber \\
&\cdot \left(\nabla_\Delta \left[ -\E_{p(\theta|S)}  \log p(z_j|\theta+\Delta) \right]_{\Delta=0} \right)^T + O\left(\frac{1}{n^2}\right) \nonumber \\
=&\left(\E_{p(\theta|S)}\nabla^2_\theta \log p(\theta,S) \right)^{-1}
\cdot \left(\E_{p(\theta|S)} \nabla_\theta  \log p(z_j|\theta) \right)^T + O\left(\frac{1}{n^2}\right) \nonumber \\
=& -\mathcal{I}_{\mathrm{MCMC}}(z_j) + O\left(\frac{1}{n^2}\right) \nonumber \\
\approx& -\mathcal{I}_{\mathrm{MCMC}}(z_j), \nonumber
\end{align}

The proof is completed.
\end{proof}

Similarly, by applying Corollary \ref{cor:bif_rm_S_prime}, the MCMC influence function of a sample set $S^\prime$ is as follows,
\begin{gather*}
    \mathcal{I}_{\mathrm{MCMC}}(S^\prime) = \sum_{z_j\in S^\prime} \mathcal{I}_{\mathrm{MCMC}}(z_j).
\end{gather*}
It guarantees that one can adopt MCMC forgetting algorithm to remove a group of datums at one time.
This improves the efficiency of removing large amounts of datums.

The MCMC forgetting algorithm also applies to SGMCMC, because SGMCMC draws samples from some target posterior $p(\theta|S)$ \citep{chen2014stochastic, welling2011bayesian, ma2015complete}, the same as MCMC.

We then give the $\varepsilon$-certified knowledge removal guarantee for MCMC forgetting.

\begin{theorem}
\label{thm:mcmc_infty_kl}
Suppose that
\begin{gather*}
p(\theta|S)=\mathcal{N}(\theta_1, (n J(\theta_1))^{-1}), \\
p(\theta|S-\{z_j\})=\mathcal{N}(\theta_2, ((n-1)J(\theta_2))^{-1}).
\end{gather*}
Let $p^{-z_j}_S(\theta)=\mathcal{A}_{\mathrm{MCMC}}(p(\theta|S),z_j)=p(\theta+\mathcal{I}_{\mathrm{MCMC}}(z)|S)$ be the processed model.
Then, $\mathcal{A}_{\mathrm{MCMC}}(p(\theta|S),z_j)$ performs $O(\varepsilon_{\theta_1,z_j})$-certified knowledge removal, where
\begin{gather*}
\varepsilon_{\theta_1,z_j} = (n-1) (\theta_1^\prime-\theta_2)^T J(\theta_2)(\theta_1^\prime -\theta_2)
+\mathrm{tr}\left(J^{-1}(\theta_1)(J(\theta_2)-J(\theta_1))\right) +\log\frac{|J(\theta_1)|}{|J(\theta_2)|}
\end{gather*}
and $\theta_1^\prime = \theta_1-\mathcal{I}_{\mathrm{MCMC}}(z_j)$.
\end{theorem}

%

Theorem \ref{thm:mcmc_infty_kl} assumes that the posteriors $p(\theta|S)$ and $p(\theta|S-\{z_j\})$ are Gaussian.
This assumption is from the Bayesian asymptotic theory \citep{gelman2013bayesian, le2012asymptotic}.
Suppose the training set $S$ is drawn from distribution $f_1(z)$, while $S-\{z_j\}$ is drawn from distribution $f_2(z)$.
Then, under some mild assumptions, when the training sample size $n$ is sufficiently large, the posteriors $p(\theta|S)$ and $p(\theta|S-\{z_j\})$ approach Gaussian distributions as below,
\begin{gather*}
    p(\theta|S)\sim \mathcal{N}(\theta_1,(n J(\theta_1))^{-1}), \\
    p(\theta|S-\{z_j\})\sim \mathcal{N}(\theta_2,((n-1)J(\theta_2))^{-1}),
\end{gather*}
where $\theta_i = \arg\min_{\theta} \mathrm{KL}(f_i(z)\|p(z|\theta))$, $i=\{1,2\}$, $J(\theta) = \E_{p(z|\theta)} \left[- \left. \frac{\partial^2}{\partial \theta^2}\log p(z|\theta) \right| \theta \right]$ is the Fisher information.


The detailed proof of Theorem \ref{thm:mcmc_infty_kl} is omitted here and given in appendix \ref{app:proof_bif_certified}.

Combining the conditions of Theorem \ref{thm:mcmc_if}, we have that $\|\theta^\prime_1-\theta_2\|_2 \leq O(1/n^2)$.
Thus, as $n\to\infty$, $\varepsilon_{\theta_1,z}$ will eventually converge to $\mathrm{tr}\left(J^{-1}(\theta_1)(J(\theta_2)-J(\theta_1))\right) +\log\frac{|J(\theta_1)|}{|J(\theta_2)|}$.

\subsection{Efficient Implementation}
\label{sec:implement}


A major computing burden in both variational inference forgetting and MCMC forgetting algorithms is calculating the product of $H^{-1}v$,
where $H$ is the Hessian matrix of some vector-valued function $f(x)$ and $v$ is a constant vector.
the calculation above would have a considerably high computational cost.
We follow {Agarwal \it et al.} \citealt{agarwal2017second} and {Koh and Liang} \citealt{koh2017understanding} to apply a divide-and-conquer strategy to address the issue. {This strategy relies on calculating the Hessian-vector product $Hv$.}

\textbf{Hessian-vector product (HVP).}
We first discuss how to efficiently calculate $Hv$.
The calculation of $Hv$ can be decomposed into two steps: (1) calculate $\frac{\partial f(x)}{\partial x}$ and then (2) calculate $\frac{\partial}{\partial x}\left( \frac{\partial f(x)}{\partial x} \cdot v \right)$.
It is worth noting that $\frac{\partial f(x)}{\partial x} \in \mathbb R^{1 \times d}$ and $v \in \mathbb R^{d \times 1}$, where $d > 0$ is the dimension of data. Thus, $\left(\frac{\partial f(x)}{\partial x} \cdot v\right)$ is a scalar value. Calculating its gradient $\frac{\partial}{\partial x}\left( \frac{\partial f(x)}{\partial x} \cdot v \right)$ has a very low computational cost on platform PyTorch \citep{paszke2017automatic} or TensorFlow \citep{tensorflow2015-whitepaper}.

\textbf{Calculating $H^{-1}v$.}
When the norm $\|H\| \leq 1$, the matrix $H^{-1}$ can be expanded by the Taylor's series as $ H^{-1}=\sum_{i=0}^\infty (I-H)^i$. Define that $H_j^{-1} = \sum_{i=0}^j (I-H)^i$. Then, we have the following recursive equation,
\begin{equation*}
H_j^{-1}v = v + (I-H)H_{j-1}^{-1}v.
\end{equation*}
{Agarwal \it et al.} \citealt{agarwal2017second} prove that when $j \rightarrow \infty$, we have $\E[H_j^{-1}] \rightarrow H^{-1}$.
Therefore, we employ $H^{-1}_j v$ to approximate $H^{-1}v$.

Moreover, to secure the condition $\|H\| \leq 1$ stands, we scale $H$ to $cH$ by a scale $c \in \R^{+}$, such that $\|cH\| \leq 1$. Then, we approximate $(cH)^{-1}$. Eventually, we have that $H^{-1}=c(cH)^{-1}$. We can plug it to the applicable equations above.

\section{Generalization Analysis}
\label{sec:generalization}
In this section, we study the generalization ability of the models that processed by forgetting algorithms.
We derive generalization bounds for a mean-field Gaussian variational model and a specified Gaussian MCMC model, Based on PAC-Bayes framework.
{All the proofs in this section are presented in Appendix \ref{app:generalization}.}

Generalization ability is important to machine learning algorithms, which refers to the ability to make accurate predictions on unseen data. A standard measurement of the generalization ability is the generalization bound, {\it i.e.}, the upper bound of the difference between expected risk and empirical risk \citep{musavi1994generalization, vapnik2013nature, mohri2018foundations}. An algorithm with a small generalization bound is expected to generalize well.
Existing generalization bound can be roughly divided into three categories:
(1) generalization bounds based on the hypothesis complexity, such as VC dimension \citep{blumer1989learnability, vapnik2006estimation}, Rademacher complexity \citep{koltchinskii2000rademacher, koltchinskii2001rademacher, bartlett2002rademacher}, and covering number \citep{dudley1967sizes, haussler1995sphere}, which suggest implementations control the hypothesis complexity to help model generalize well;
(2) generalization bounds based on the algorithmic stability \citep{rogers1978finite, bousquet2002stability}, which follow the intuition that an algorithm with good generalization ability is robust to the interference of single data points; (3) generalization bounds established under the PAC-Bayesian framework \citep{mcallester1999pac, mcallester1998some}; and (4) generalization guarantees from differential privacy \citep{dwork2015preserving, oneto2017differential, he2020tighter}.
The excellent generalization ability of the over-parameterized model, including Bayesian neural network and others in deep learning, is somehow beyond the explanation of the conventional learning theory. Establishing theoretical foundations has been attracted wide attention \citep{e2020towards, he2020recent}.

Let $Q$ be a probabilistic model.
Suppose $S$ is the training sample set.
Then, the {\it expected risk} $\mathcal{R}(Q)$ and {\it empirical risk} $\mathcal{\hat R}(Q,S)$ of $Q$ are defined to be
\begin{gather*}
    \mathcal{R}(Q) = \mathop{\E}_{h \sim Q} \mathop{\E}_{z} \ell(h, z), \ \ \ \
    \mathcal{\hat R}(Q,S) = \mathop{\E}_{h \sim Q} \frac{1}{n}\sum_{i=1}^n\ell(h, z_i),
\end{gather*}
where $h$ is a hypothesis drawn from $Q$, and $\ell$ is the loss function ranging in $[0,1]$.
The difference of expected risk and empirical risk is the generalization error. Its magnitude characterizes the generalizability of the algorithm.

We then prove a generalization bound for mean-field Gaussian variational distributions.
\begin{theorem}
\label{thm:gen_if}
Suppose all the conditions in Theorems \ref{thm:vi_bif} and \ref{thm:vi_certified} hold.
Let $q_{\lambda}=\mathcal{N}(\mu,\sigma^2 I)$ denotes the mean-field Gaussian distribution learned on the training set $S$.
Let $\lambda^- = \mathcal{A}_{\mathrm{VI}}(\lambda,z_j)$ denotes the processed variational distribution parameter, where $z_j \in S$.
Also, let $\Delta_{\lambda} = \mathcal{I}_{\mathrm{VI}}(z_j) = (\Delta_\mu,\Delta_\sigma)$ denotes the variational inference influence function.
Then, for any real $\delta \in (0,1)$, with probability at least $1-\delta$, the following inequality holds:
\begin{gather}
    \mathcal{R}(q_{\lambda^-})
    \leq \mathcal{\hat R}(q_{\lambda^-},S)
         + \sqrt{\frac{C_{\lambda,\Delta_\lambda} + 2\log\frac{1}{\delta}+2\log n -d + 4}{4n-2}},
    \label{equ:gaussian_vi_bound}
\end{gather}
where
\begin{gather*}
C_{\lambda,\Delta_{\lambda}}=\|\Delta_{\lambda}\|^2 + 2 \|\lambda\| \cdot \|\Delta_{\lambda}\| + \|\lambda\|^2
- 2 \sum_{k=1}^d\log\left(\sigma_k-\Delta_{\sigma_k}\right)
\leq O(1),
\end{gather*}
and $\|\Delta_{\lambda}\| \leq O\left(\frac{1}{n}\right)$.
\end{theorem}

\begin{remark}
This generalization bound is of order $O(\sqrt{(\log n) / n})$.
\end{remark}

%

\begin{corollary}
Variational inference forgetting increases the generalization bound by a value not larger than $O(1/n)$.
\end{corollary}

\begin{proof}
Variational inference forgetting introduces the term $\|\Delta_{\lambda}\|$ into the generalization bound.
Thus, the generalization bound is increased by the following value:
\begin{align*}
& O\left( \sqrt{ \frac{\|\Delta_{\lambda}\|^2 + 2 \|\lambda\| \cdot \|\Delta_{\lambda}\| - 2 \sum_{k=1}^d\log\left( \frac{\sigma_k-\Delta_{\sigma_k}}{\sigma_k} \right)}{4n-2} } \right) \\
&\quad \leq O\left( \sqrt{ \frac{O(1/n^2) + O(1/n) + O(1/n)}{4n-2} } \right)
= O\left( \frac{1}{n} \right).
\end{align*}

The proof is completed.
\end{proof}

This corollary secures that variational inference forgetting would not compromise the generalizability.


In Section \ref{sec:mcmc_if}, we have shown that when the training sample size is sufficiently large, the posterior distribution $p(\theta|S)$ is asymptotically Gaussian.
Here, we again assume that $p(\theta|S) = \mathcal{N}(\theta_1,(nJ(\theta_1))^{-1})$, where $J(\theta)$ is the fisher information matrix.
Then, we obtain a generalization bound as follows.
\begin{theorem}
\label{thm:gen_mcmc}
Suppose all the conditions in Theorem \ref{thm:mcmc_if} hold.
Suppose that $p(\theta|S) = \mathcal{N}(\theta_1,(n J(\theta_1))^{-1})$.
Let $p^-=\mathcal{A}_{\mathrm{MCMC}}(p(\theta|S),z_j)$ denotes the processed distribution, where $z_j \in S$.
Let $\Delta_\theta = \mathcal{I}_{\mathrm{MCMC}}(z_j)$ denotes the MCMC influence function.
Then, for any $\delta \in (0,1)$, with probability at least $1-\delta$, the following inequality holds:
\begin{align}
    \mathcal{R}(p^-)
    \leq \mathcal{\hat R}(p^-,S)
         + \sqrt{\frac{C_{p,\Delta_{\theta_1}} + 2\log\frac{1}{\delta}+(d+2)\log n -d+ 4}{4n-2}},
    \label{equ:gaussian_mcmc_bound}
\end{align}
where
\begin{align*}
C_{p,\Delta_{\theta_1}}=\|\Delta_{\theta_1}\|^2 + 2\|\theta_1\| \cdot \|\Delta_{\theta_1}\|
+ \|\theta_1\|^2 + \frac{1}{n}\mathrm{tr}(J^{-1}(\theta_1)) + \log|J(\theta_1)|
= O(1),
\end{align*}
and $\|\Delta_{\theta_1}\| \leq O\left(\frac{1}{n}\right)$.
\end{theorem}

\begin{remark}
This generalization bound is of order $O(\sqrt{(\log n) / n})$.
\end{remark}

%
%
%

\begin{corollary}
MCMC forgetting increases the generalization bound by a value not larger than $O(1/n)$.
\end{corollary}

\begin{proof}
MCMC forgetting introduces the term $\|\Delta_{\theta_1}\|$ into the generalization bound.
Thus, the generalization bound is increased by the following value:
\begin{gather*}
O\left( \sqrt{ \frac{\|\Delta_{\theta_1}\|^2 + 2\|\theta_1\| \cdot \|\Delta_{\theta_1}\|}{4n-2} } \right) \leq O\left( \sqrt{\frac{O(1/n^2) + O(1/n)}{4n-2}} \right) = O\left( \frac{1}{n} \right).
\end{gather*}

The proof is completed.
\end{proof}

This corollary secures that MCMC forgetting would not compromise the generalizability.


\section{Experiments}
\label{sec:experiments}

We apply our forgetting algorithm to a Gaussian mixture model and a Bayesian neural network.
In every scenario, we employ variational inference and two SGMCMC methods, SGLD and SGHMC.
The empirical results are in full agreement with our methods.
All the experiments are conducted on a computer with a GPU of NVIDIA\textsuperscript{\textregistered} GeForce\textsuperscript{\textregistered} RTX 2080 Ti, a CPU of Intel\textsuperscript{\textregistered} Core\textsuperscript{\texttrademark} i9-9900 @ 3.10GHz, and 32GB memory.
To secure reproducibility, the source code package is available at \url{https://github.com/fshp971/BIF}.

\subsection{Experiments for Gaussian Mixture Model}
\label{exp:gmm}

We first conduct experiments with GMM on Synthetic data to evaluate our methods.

\subsubsection{Implementation Details}

The implementation details are given below.

\textbf{Synthetic dataset.}
We generate a dataset $S$ of size $2,000$ for evaluating our algorithms. Every datum is two-dimensional and is possibly from $K$ classes. In this experiment, we set $K$ as $4$. The raw data is visualized in the left of fig. \ref{fig:gmm_raw}, in which four different colors represent four different classes.

{\textbf{Gaussian mixture model (GMM).}
GMMs are usually employed to inference the cluster centers.
GMM assumes data is drawn from $K$ Gaussian distributions centered at $\mu_1, \cdots, \mu_K$, respectively.
The hierarchical structure of GMM is as follows: (1) draw a clustering center from the uniform distribution over $\{\mu_{c_1}, \ldots, \mu_{c_K}\}$; and (2) sample $z_i$ from a Gaussian distribution centering at $\mu_{c_i}$.
\begin{align*}
    \mu_k &\sim \mathcal{N}(0, \sigma^2 I), \\
    c_i &\sim \mathrm{categorical}\left(\frac{1}{K}, \cdots, \frac{1}{K}\right), \\
    Z_i &\sim \mathcal{N}(\mu_{c_i}, I),
\end{align*}
where $1\leq k \leq K$, $1 \leq i \leq n$, $\mu_k \in \R^d$, $c_i \in \{1,\cdots,K\}$, $Z_i \in \R^d$, and the hyperparameter $\sigma \in \R$ is the prior standard deviation. 
We set $\sigma$ as $1$ in our experiments.
}

Applying SGLD and SGHMC to GMM is straightforward. Besides, for variational inference, we utilize the following mean-field variational family \citep{blei2017variational},
\begin{align*}
    q(\bm \mu, \bm c) &= \prod_{k=1}^K q(\mu_k) \prod_{i=1}^n q(c_i|\mu), \\
    \mu_k &\sim \mathcal{N}(m_k, s_k^2), \\
    q(c_i=k|\mu)
    &\propto \exp \left( \E_{q_{-c_i}} \log p(x_i|\mu, c_i) \right) \nonumber \\
    &= \exp \sum_j \left( x_{ij} m_{kj} - \frac{m_{kj}^2 + s_{kj}^2}{2} \right),
    \label{equ:gmm_svi}
\end{align*}
where $x_i, m_k, s_k \in \R^d$, and $\lambda=(m_1,\cdots,m_K, s_1, \cdots, s_K)$ is the variational distribution parameter. Thus, the ELBO function is calculated as follows,
\begin{align*}
\mathrm{ELBO}(\lambda,S)
&= -\sum_{k,j} \frac{m_{kj}^2+s_{kj}^2}{2\sigma^2} + \sum_{k,j} \log s_{kj}
- \sum_{i,k,j} \varphi_{ik} \frac{- 2 x_{ij} m_{kj} + m_{kj}^2 + s_{kj}^2}{2} \nonumber \\
& \ \ \ \ -\sum_{i,k} \varphi_{ik} \log \varphi_{ik}
 - \frac{1}{2}\sum_{i,j} x_{ij}^2 + \mathrm{Constant},
\end{align*}
where $1\leq i \leq n$, $1\leq k \leq K$, $1\leq j \leq d$, and
\begin{gather*}
\varphi_{ik} = q(c_i=k|\mu) \propto \exp \sum_j \left( x_{ij} m_{kj} - \frac{m_{kj}^2 + s_{kj}^2}{2} \right).
\end{gather*}


\begin{figure}[!t]
    \begin{subfigure}{0.49\linewidth}
        \center
        \includegraphics[width=0.49\linewidth]{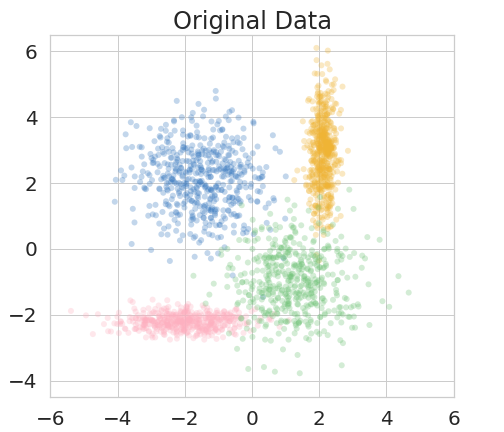}
        \includegraphics[width=0.49\linewidth]{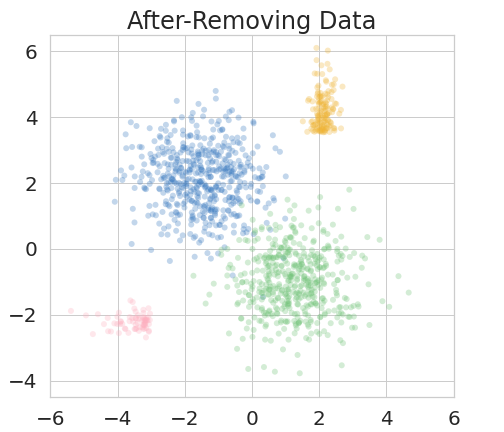}
        \caption{Visualization of dataset $S$.}
        \label{fig:gmm_raw}
    \end{subfigure}
    \begin{subfigure}{0.49\linewidth}
        \center
        \includegraphics[width=0.49\linewidth]{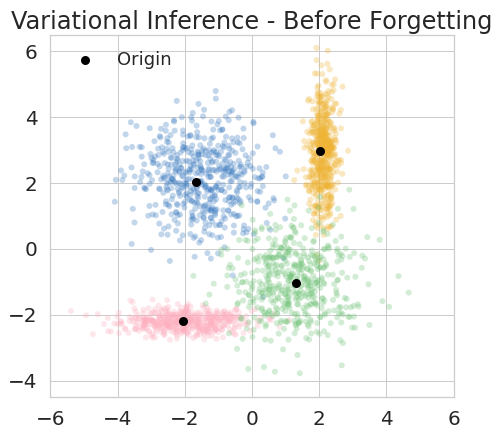}
        \includegraphics[width=0.49\linewidth]{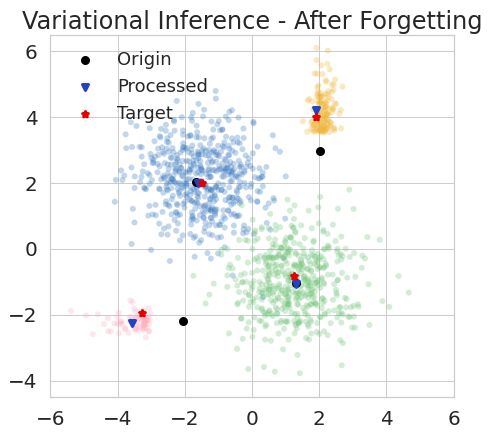}
        \caption{Results of variational inference.}
        \label{fig:gmm_svi}
    \end{subfigure}

    \begin{subfigure}{0.49\linewidth}
        \center
        \includegraphics[width=0.49\linewidth]{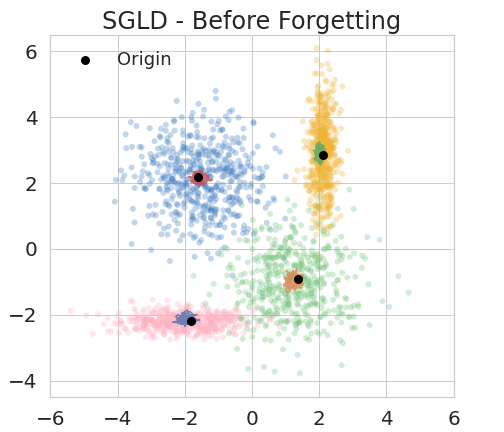}
        \includegraphics[width=0.49\linewidth]{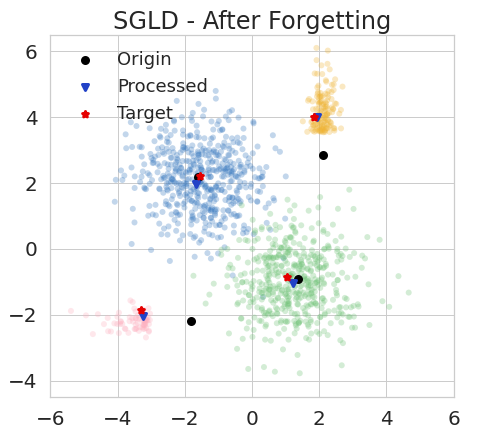}
        \caption{Results of SGLD.}
        \label{fig:gmm_sgld}
    \end{subfigure}
    \begin{subfigure}{0.49\linewidth}
        \center
        \includegraphics[width=0.49\linewidth]{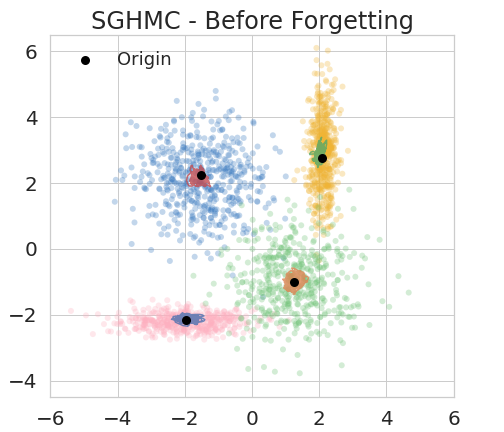}
        \includegraphics[width=0.49\linewidth]{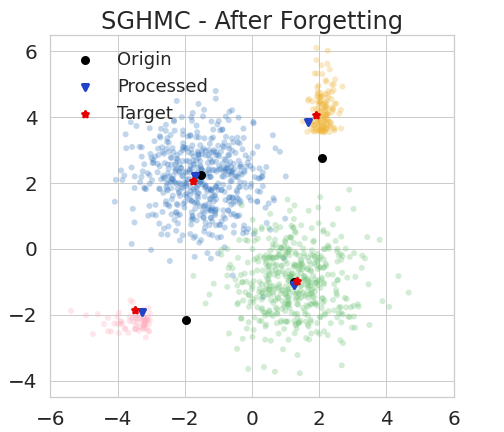}
        \caption{Results of SGHMC.}
        \label{fig:gmm_sghmc}
    \end{subfigure}
    \caption{
        Visualized results of GMM experiments. There are four groups of points in the synthetic dataset, where different groups are in different colors.
        The clustering results that trained on the complete dataset are shown in the left parts of figs. \ref{fig:gmm_svi}, \ref{fig:gmm_sgld}, \ref{fig:gmm_sghmc}.
        The results of the origin centers, processed centers, and target centers are drawn in black points, blue triangles, and red asterisks in the right parts of figs. \ref{fig:gmm_svi}, \ref{fig:gmm_sgld}, \ref{fig:gmm_sghmc}, where the origin centers are obtained via training on the complete set, the processed centers are obtained after BIF, and the target centers are obtained via training on only the remaining set.}
    \label{fig:gmm}
\end{figure}

\textbf{Experiment design.} We employ variational inference, SGLD, and SGHMC to inference GMM on the synthetic datums. Then, we remove $400$ points from each of the pink parts and yellow parts, around $40\%$ of the whole dataset at all, by the proposed BIF algorithms. The remaining datums are shown in the right of fig. \ref{fig:gmm_raw}. We also trained models on only the remaining set with the same Bayesian inference settings in order to show the targets of the forgetting task. The experiments have two main phrases: 

{
(1) {\it Training phase}. Every GMM is trained for $2,000$ iterations. The batch size is set as $64$. For variational inference, the learning rate is fixed to $2/n$, where $n$ is the training sample set size. For SGLD, the learning rate schedule is set as $4 \cdot t^{-0.15} / n$, where $t$ is the training iteration step. For SGHMC, the learning rate schedule is set as $2 \cdot t^{-0.15} / n$, and the initial $\alpha$ factor is set as $0.4$.


(2) {\it Forgetting phase}.
We remove a batch of $4$ datums each time.
When calculating the inversed-Hessian-vector product $H^{-1}v$ in the influence functions (see Section \ref{sec:implement}), the recursive calculation number $j$ is set as $32$, and the scaling factor $c$ is set as $1/n^\prime$, where $n^\prime$ is the number of the current remaining training datums. Notice that $n^\prime$ will gradually decrease as the forgetting process continuing.
Moreover, for SGLD and SGHMC, we employ Monte Carlo method to calculate the expectations in MCMC influence functions.
Specifically, we repeatedly sample model parameter $\theta$ for $5$ times, calculate the matrix or vector in MCMC influence functions, and average the results to approach the expectations.

}

\subsubsection{Results Analysis}

The empirical results are presented in figs. \ref{fig:gmm_svi}, \ref{fig:gmm_sgld}, \ref{fig:gmm_sghmc} respectively. In figures for variational inference, we draw the learned clustering centers as black points. In every figure for SGLD and SGHMC, (1) we draw the point drawn in the terminated iteration as a black point; and (2) we draw the points drawn in the last $500$ iterations in blue, orange, red, or green. The results obtained on the remaining parts (target) are also shown in these figures.
This visualization shows that after forgetting, the learned models are close to the target models,
which demonstrate that our forgetting algorithms can selectively remove specified datums while protecting others intact.

\subsection{Experiments for Bayesian Neural Network}
\label{sec:exp_bnn}

We then conduct experiments with Bayesian neural networks on real data.

\subsubsection{Implementation Details}

The implementation details are given below.

\textbf{Dataset.}
We employ Fashion-MNIST \citep{xiao2017/online} dataset in our experiments. It consists of $28 \times 28$ gray-scale images from $10$ different classes, where each class consists of $6,000$ training examples and $1,000$ test examples. For the data argumentation, we first resize each image to $32 \times 32$ and then normalize each pixel value to $[-0.5,0.5]$ before feeding them into BNN.
For the forgetting experiments, we divide the training set of Fashion-MNIST into two parts, the removed part $S_f$ and the remaining part $S_r$. Apparently, $S_f \cup S_r = S$. For the selection of the removed training set $S_f$, we randomly choose $1,000$, $2,000$, $3,000$, $4,000$, $5,000$ and $6,000$ examples from the class ``T-shirt'' ({\it i.e.}, examples that labeled with number $0$) in the training set to form $S_f$. For the brevity, we denote the test set by $S_{\text{test}}$.

\textbf{Bayesian neural network (BNN).}
BNNs employ Bayesian inference to inference the posterior of the neural networks parameters.
Two major Bayesian inference methods employed wherein are variational inference and SGMCMC.

For variational inference, one usually utilizes the mean-field Gaussian variational family \citep{blundell2015weight, kingma2015variational} to train BNNs. 
The {\it Bayes by Backprop} technique \citep{blundell2015weight} is utilized to calculate the derivative of ELBO function.
Specifically, for the random variable $\theta$ subject to the mean-field variational distribution $q_\lambda=\mathcal{N}(\mu,\sigma^2 I)$, we have that $(\theta-\mu)/\sigma \sim \mathcal{N}(0,I)$. Let $\varepsilon=(\theta-\mu)/\sigma$, then the derivative of the ELBO function can be varied as follows,
\begin{align*}
& \nabla_\lambda \mathrm{ELBO}(\lambda,S) \\
& \quad = \nabla_\lambda \E_{q_\lambda} \left(\log p(\theta,S)-q_\lambda(\theta) \right) \\
& \quad = \nabla_\lambda \E_{\varepsilon} \left(\log p(\theta,S)-q_\lambda(\theta) \right) \\
& \quad = \E_{\varepsilon} \nabla_\lambda \left(\log p(\theta,S)-q_\lambda(\theta) \right),
\end{align*}
Hence, we can calculate the derivative $\nabla_\lambda\mathrm{ELBO}(\lambda,S)$ in two steps: (1) repeatedly sample $\varepsilon$ from $\mathcal{N}(0,I)$ and calculate $\nabla_\lambda \left(\log p(\theta,S)-q_\lambda(\theta) \right)$ based on $\theta=\mu + \varepsilon \cdot \sigma$; and (2) average the obtained derivatives to approximate $\nabla_\lambda\mathrm{ELBO}(\lambda,S)$. Moreover, we also adopt the {\it local reparameterization trick} \citep{kingma2015variational} to further eliminate the covariances between the gradients of examples in a batch.

{
Bayesian LeNet-5 \citep{lecun1998gradient} is employed in our experiments, which consists of two convolutional layers and three fully-connected layers.
We follow {Liu \it et al.} \citet{liu2018adv} to use an isotropic Gaussian distribution $\mathcal{N}(0,\sigma^2 I)$ as the prior of BNN, where the standard deviation $\sigma$ is set as $0.15$.}


\textbf{Experiment design.} We employ variational inference, SGLD, and SGHMC to train the BNN on the complete training set $S$. Then, we remove the subset $S_f$ by the proposed BIF algorithms. We also trained models on only the remaining set $S_r$ in order to show the targets of the forgetting task. The experiments have two main phrases:

{
(1) {\it Training phase}. Every BNN is trained for $10,000$ iterations. The batch size is set as $128$. For variational inference, the learning rate is initialized as $0.5/n$, where $n$ is the training set size, and decay by $0.1$ every $4,000$ iterations.
Additionally, the sampling times to perform Bayes by Backprop procedure is set as $5$.
For SGLD, the step-size schedule is set as $0.5 \cdot t^{-0.5} / n$, where $t$ is the training iteration step. For SGHMC, the step-size schedule is set as $0.5 \cdot t^{-0.5} / n$, and the initial $\alpha$ factor is set as $0.4$.}


{
(2) {\it Forgetting phase}.
We remove a batch of $64$ datums each time.
When calculating the inversed-Hessian-vector product $H^{-1}v$ in influence functions (see Section \ref{sec:implement}), the recursive calculation number $j$ is set as $64$.
For variational inference, the scaling factor $c$ is set as $0.1/n^\prime$, where $n^\prime$ is the number of the current remaining datums.
Notice that $n^\prime$ will gradually decrease as the forgetting process continuing.
For SGLD and SGHMC, the scaling factors $c$ are set as $0.005/n^\prime$ and $0.05/n^\prime$, respectively.
Besides, we employ the Monte Carlo method to calculate the expectations in MCMC influence functions.
Specifically, we repeatedly sample model parameter $\theta$ for $5$ times, 
calculate the matrix and vector in MCMC influence function based on these sampled parameters,
and average the results to approach the desired expectation.}


\begin{figure}[!t]
    \begin{subfigure}{\linewidth}
        \includegraphics[width=0.31\linewidth]{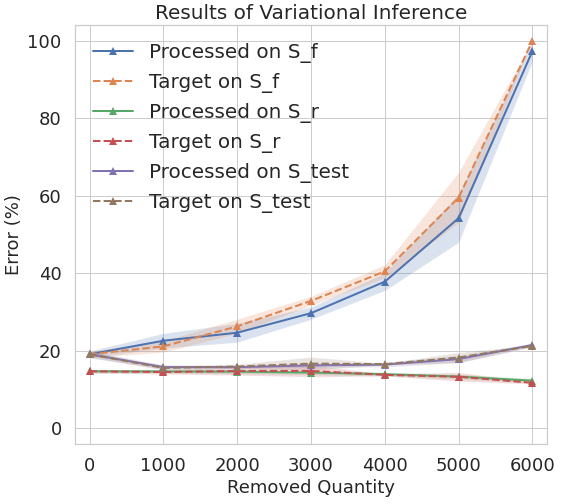}
        \hspace{0.5em}
        \includegraphics[width=0.31\linewidth]{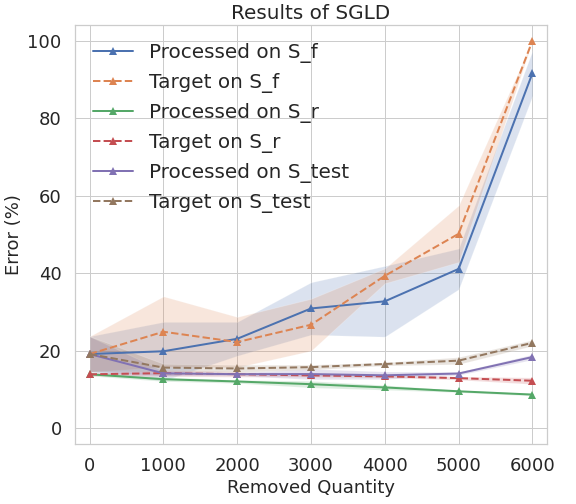}
        \hspace{0.5em}
        \includegraphics[width=0.31\linewidth]{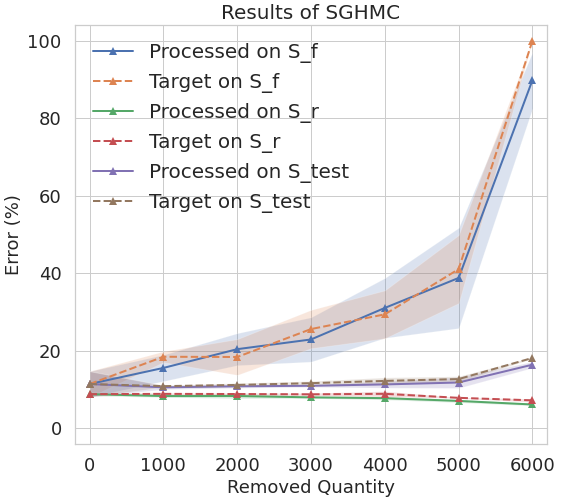}
    \end{subfigure}
    \caption{
        Curves of classification error to the number of the removed training datums. 
        The results of variational inference, SGLD, and SGHMC are presented from left to right, respectively. For each setting, three classification error curves on the forgetted set $S_f$, remained set $S_r$, and test set $S_{\text{test}}$ are plotted. The models are trained for $5$ times with different random seeds. The darker lines show the average over seeds and the shaded area shows the standard deviations. }
    \label{fig:exp_bnn}
\end{figure}

\begin{table}[t]
\centering
\caption{
Time of the training phase and the forgetting phase in the experiments for BNNs.
The acceleration rate is calculated by dividing ``training time'' with ``forgetting time''.
Our forgetting algorithms is significantly faster than re-training from scratch.}
 \begin{tabular}{c c c c}
\toprule
\multirow{2}{2.5cm}{} & \multirow{2}{3cm}{\centering Variational Inference} & \multirow{2}{3cm}{\centering SGLD} & \multirow{2}{3cm}{\centering SGHMC} \\ \\
\midrule
Training Time & $337.77$s & $39.36$s & $39.90$s \\
Removal Time & $4.97$s & $4.04$s & $4.04$s \\
\midrule
Acceleration Rate & $67.90$ & $9.75$ & $9.88$ \\
\bottomrule
\end{tabular}
\label{tab:fashion_time}
\end{table}

\subsubsection{Results Analysis}

For each of the obtained processed models and target models, we evaluate its classification errors on the sets $S_f$, $S_r$ and $S_{\text{test}}$, respectively. The results are collected and plotted in fig. \ref{fig:exp_bnn}.
We also collect and present the training and forgetting times in Table \ref{tab:fashion_time}.

From fig. \ref{fig:exp_bnn} and Table \ref{tab:fashion_time}, we have the following two observations:
(1) {the processed models} are similar to the target models in terms of the classification errors on all the three sample sets $S_f$, $S_r$, and $S_{\text{test}}$;
and (2) our forgetting algorithms are significantly faster than simply training models from scratch.
These phenomena demonstrate that our forgetting algorithms can effectively and efficiently remove influences of datums from BNNs without hurting other remaining information.

\section{Conclusion}
\label{sec:conclusion}
The right to be forgotten imposes a considerable compliance burden on AI companies. A company may need to delete the whole model learned from massive resources due to a request to delete a single datum. To address this problem, this work designs a {\it Bayesian inference forgetting} (BIF) that removes the influence of some specific datums on the learned model without completely deleting the whole model.
The BIF framework is established on an energy-based Bayesian inference influence function, which characterizes the influence of requested data on the learned models. We prove that BIF has an $\varepsilon$-certified knowledge removal guarantee, which is a new term on characterizing the forgetting performance. 
Under the BIF framework, forgetting algorithms are developed for two canonical Bayesian inference algorithms, variational inference and Markov chain Monte Carlo.
Theoretical analysis provides guarantees on the generalizability of the proposed methods: performing the proposed BIF has little effect on them.
Comprehensive experiments demonstrate that the proposed methods can remove the influence of specified datums without compromising the knowledge learned on the remained datums.
The source code package is available at \url{https://github.com/fshp971/BIF}.


\bibliography{BIF}

\appendix



\section{Proofs in Section \ref{sec:main_bif}}
\label{app:proof_bif_framework}

This section provides the missing proofs in Section \ref{sec:main_bif}.

\begin{proof}[Proof of Theorem \ref{thm:1st_norm}]
We first calculate $\frac{\partial\hat\gamma^{-z_j}_S(\tau)}{\partial\tau}$ based on eq. (\ref{equ:prob_diff}), {\it i.e.}, the following equation,
\begin{gather*}
    \nabla_{\gamma}F(\gamma,S) + \tau \cdot \nabla_{\gamma}h(\gamma,z_j) = 0.
\end{gather*}
Calculate the derivatives of the both sides of the above equation with respect to $\tau$, we have that
\begin{gather}
    \nabla^2_\gamma F(\hat\gamma^{-z_j}_S(\tau),S)\cdot\frac{\partial\hat\gamma^{-z_j}_S(\tau)}{\partial\tau} + \nabla_\gamma h(\hat\gamma^{-z_j}_S(\tau),z_j)^T + \tau \cdot \nabla^2_\gamma h(\hat\gamma^{-z_j}_S(\tau),z_j) \cdot \frac{\partial\hat\gamma^{-z_j}_S(\tau)}{\partial\tau} = 0.
    \label{thm:1st_norm:equ:diff_equ}
\end{gather}
By lemma \ref{lem:scox_eps}, $F_{-z_j,\tau}(\gamma,S)$ is strongly convex with respect to $\gamma$. Thus, the following Hessian matrix
\begin{gather*}
    \nabla^2_\gamma F_{-z_j,\tau}(\hat\gamma^{-z_j}_S(\tau),S)
    =\nabla^2_\gamma F(\hat\gamma^{-z_j}_S(\tau),S)
    +\tau\cdot\nabla^2_\gamma h(\hat\gamma^{-z_j}_S(\tau),z_j)
\end{gather*}
is positive definite, hence invertible. Combining with eq. (\ref{thm:1st_norm:equ:diff_equ}), we have that
\begin{align}
    \frac{\partial\hat\gamma^{-z_j}_S(\tau)}{\partial\tau}
    =& -\left(\nabla^2_\gamma F(\hat\gamma^{-z_j}_S(\tau),S) + \tau\cdot\nabla^2_\gamma h(\hat\gamma^{-z_j}_S(\tau),z_j)\right)^{-1}
    \cdot \nabla_\gamma h(\hat\gamma^{-z_j}_S(\tau),z_j)^T.
    \label{thm:1st_norm:equ:diff}
\end{align}

We then upper bound the norm of $\frac{\partial\hat\gamma^{-z_j}_S(\tau)}{\partial\tau}$. Based on eq. (\ref{thm:1st_norm:equ:diff}), we have that
\begin{align}
    \left\| \frac{\partial\hat\gamma^{-z_j}_S(\tau)}{\partial\tau} \right\|_2
    \leq& \left\| \left(\frac{1}{n}\nabla^2_\gamma F(\hat\gamma^{-z_j}_S(\tau),S) + \frac{\tau}{n}\nabla^2_\gamma h(\hat\gamma^{-z_j}_S(\tau),z_j)\right)^{-1} \right\|_2
    \cdot \left\| \frac{1}{n}\nabla_\gamma h(\hat\gamma^{-z_j}_S(\tau),z_j) \right\|_2.
    \label{thm:1st_norm:equ:upp_bd}
\end{align}

We first consider the first term of the right-hand side of eq. (\ref{thm:1st_norm:equ:upp_bd}).
By assumption \ref{ass:cont_diff_scox}, $f(\gamma)$ is $c_f$-strongly convex and $h(\gamma,z)$ is $c_h(z)$-strongly convex, where $c_f$ is a positive real number, $c_h(z)$ is a positive continuous real function on $\mathcal{Z}$. Thus we have that
\begin{align*}
    &\frac{1}{n}\nabla^2_\gamma F(\hat\gamma^{-z_j}_S(\tau),S) + \frac{\tau}{n}\nabla^2_\gamma h(\hat\gamma^{-z_j}_S(\tau),z_j) \\
    =& \frac{1}{n}\sum_{i=1}^n \nabla^2_\gamma h(\hat\gamma^{-z_j}_S(\tau),z_i) + \frac{1}{n}\nabla^2_\gamma f(\hat\gamma^{-z_j}_S(\tau)) + \frac{\tau}{n}\nabla^2_\gamma h(\hat\gamma^{-z_j}_S(\tau),z_j) \\
    =& \frac{1}{n}\sum_{z\in S-\{z_j\}} \nabla^2_\gamma h(\hat\gamma^{-z_j}_S(\tau),z) + \frac{1+\tau}{n} \nabla^2_\gamma h(\hat\gamma^{-z_j}_S(\tau),z_j) + \frac{1}{n} \nabla^2_\gamma f(\hat\gamma^{-z_j}_S(\tau)) \\
    \succeq& \left( \frac{1}{n} \sum_{z\in S-\{z_j\}} c_h(z) + \frac{1+\tau}{n}c_h(z_j) + \frac{c_f}{n} \right) I \\
    \succeq& \left( \frac{1}{n} \sum_{z\in S-\{z_j\}} c_h(z) \right) I.
\end{align*}
Applying Assumption \ref{ass:compact}, we have that $c_h(z)$ is continuous on the compact set $\mathcal{Z}$.
This suggests that there exists a real constant $\hat c_h>0$ such that $\forall z\in\mathcal{Z}$, $c_h(z)\geq \hat c_h$. Therefore, we further have that
\begin{gather*}
    \frac{1}{n}\nabla^2_\gamma F(\hat\gamma^{-z_j}_S(\tau),S) + \frac{\tau}{n}\nabla^2_\gamma h(\hat\gamma^{-z_j}_S(\tau),z_j)
    \succeq \left( \frac{1}{n} \sum_{z\in S-\{z_j\}} c_h(z) \right) I
    \succeq \left( \frac{n-1}{n} \hat c_h \right) I.
\end{gather*}
Let $\lambda_{\min}$ denotes the smallest eigenvalue of the matrix
$\left(\frac{1}{n}\nabla^2_\gamma F(\hat\gamma^{-z_j}_S(\tau),S) + \frac{\tau}{n}\nabla^2_\gamma h(\hat\gamma^{-z_j}_S(\tau),z_j)\right)$.
Then, the above inequality implies that
$\lambda_{\min} \geq \frac{n-1}{n} \hat c_h$.
Hence, we have the following,
\begin{gather}
\left\| \left(\frac{1}{n}\nabla^2_\gamma F(\hat\gamma_S(\tau),S) + \frac{\tau}{n}\nabla^2_\gamma h(\hat\gamma_S(\tau),z_j)\right)^{-1} \right\|_2
= \frac{1}{\lambda_{\min}}
\leq \frac{n}{(n-1) \hat c_h} = O(1).
\label{thm:1st_norm:equ:upp_term_1}
\end{gather}

We then upper bound the second term of the right-hand side of eq. (\ref{thm:1st_norm:equ:upp_bd}). 
Applying Assumptions \ref{ass:cont_diff_scox} and \ref{ass:compact}, we have that $\nabla^2_\gamma h(\gamma,z)$ is continuous on compact support $\Gamma\times\mathcal{Z}$.
This demonstrates that both $\nabla^2_\gamma h(\gamma,z)$ and $\left\|\nabla^2_\gamma h(\gamma,z)\right\|_2$ are bounded.
Therefore, as $n\to\infty$, we have
\begin{gather}
    \left\|\frac{1}{n}\sum_{i=1}^n \nabla^2_\gamma h(\hat\gamma_S(\tau),z_i)\right\|_2
    = \frac{1}{n} \left\|\sum_{i=1}^n \nabla^2_\gamma h(\hat\gamma_S(\tau),z_i)\right\|_2
    \leq O\left(\frac{1}{n}\right).
\label{thm:1st_norm:equ:upp_term_2}
\end{gather}

Finally, inserting eqs. (\ref{thm:1st_norm:equ:upp_term_1}) (\ref{thm:1st_norm:equ:upp_term_2}) into eq. (\ref{thm:1st_norm:equ:upp_bd}), we eventually have that
\begin{gather*}
    \left\| \frac{\partial\hat\gamma_S(\tau)}{\partial\tau} \right\|_2
    \leq O(1) \cdot O\left(\frac{1}{n}\right) = O\left(\frac{1}{n}\right).
\end{gather*}

The proof is completed.
\end{proof}

\begin{proof}[Proof of Theorem \ref{thm:2nd_norm}]
We first calculate $\frac{\partial^2\gamma^{-z_j}_S(\tau)}{\partial\tau^2}$ based on eq. (\ref{equ:prob_diff}). Similar to the proof of Theorem \ref{thm:1st_norm}, we calculate the second-order derivatives of the both sides of eq. (\ref{equ:prob_diff}) with respect to $\tau$ and have that
\begin{align*}
    &\sum_{i=1}^n B(\tau,z_i) + A(\tau)
    + \left(\sum_{i=1}^n \nabla^2_\gamma h(\hat\gamma^{-z_j}_S(\tau),z_i) + \nabla^2_\gamma f(\hat\gamma^{-z_j}_S(\tau))\right)\cdot\frac{\partial^2 \hat\gamma^{-z_j}_S(\tau)}{\partial\tau^2} \\
    &+ 2 \cdot \nabla^2_\gamma h(\hat\gamma^{-z_j}_S(\tau),z_j) \cdot \frac{\partial\hat\gamma^{-z_j}_S(\tau)}{\partial\tau}
    + \tau \cdot B(\tau,z_j)
    + \tau \cdot \nabla^2_\gamma h(\hat\gamma^{-z_j}_S(\tau),z_j) \cdot\frac{\partial^2 \hat\gamma^{-z_j}_S(\tau)}{\partial\tau^2}  = 0,
\end{align*}
which means
\begin{align}
    \frac{\partial^2 \hat\gamma^{-z_j}_S(\tau)}{\partial\tau^2}
    =& -\left(\frac{1}{n}\sum_{i=1}^n \nabla^2_\gamma F(\hat\gamma^{-z_j}_S(\tau),S) + \frac{\tau}{n}\nabla^2_\gamma h(\hat\gamma^{-z_j}_S(\tau),z_j)\right)^{-1} \nonumber \\
    &\cdot \left(\frac{1}{n}\sum_{i=1}^n B(\tau,z_i) + \frac{\tau}{n} B(\tau,z_j) + \frac{1}{n} A(\tau) + \frac{2}{n} \nabla^2_\gamma h(\hat\gamma^{-z_j}_S(\tau),z_j) \cdot \frac{\partial\hat\gamma^{-z_j}_S(\tau)}{\partial\tau} \right),
    \label{thm:2nd_norm:equ:eq_1}
\end{align}
where the invertibility of
$\left(\frac{1}{n}\sum_{i=1}^n \nabla^2_\gamma F(\hat\gamma^{-z_j}_S(\tau),S) + \frac{\tau}{n}\nabla^2_\gamma h(\hat\gamma^{-z_j}_S(\tau),z_j)\right)$
is guaranteed by Lemma \ref{lem:scox_eps}, $A(\tau), B(\tau,z) \in \R^{K\times 1}$, and for $i = 1, \cdots, K$, we have the following,
\begin{gather}
    A(\tau)_i = {\frac{\partial\hat\gamma^{-z_j}_S(\tau)}{\partial\tau}}^T \cdot \nabla^2_\gamma \left( \frac{\partial f(\hat\gamma^{-z_j}_S(\tau))}{\partial \gamma_i} \right) \cdot \frac{\partial\hat\gamma^{-z_j}_S(\tau)}{\partial\tau}, \label{thm:2nd_norm:equ:A_elem} \\
    B(\tau,z)_i = {\frac{\partial\hat\gamma^{-z_j}_S(\tau)}{\partial\tau}}^T \cdot \nabla^2_\gamma \left( \frac{\partial h(\hat\gamma^{-z_j}_S(\tau),z)}{\partial \gamma_i} \right) \cdot \frac{\partial\hat\gamma^{-z_j}_S(\tau)}{\partial\tau}.
\end{gather}

We then upper bound the norm of
$\frac{\partial^2 \hat\gamma^{-z_j}_S(\tau)}{\partial\tau^2}$.
Based on eq. (\ref{thm:2nd_norm:equ:eq_1}), we have that
\begin{align}
    &\left\| \frac{\partial^2 \hat\gamma^{-z_j}_S(\tau)}{\partial\tau^2} \right\|_2 \nonumber \\
    & \leq \left\| \left(\frac{1}{n}\sum_{i=1}^n \nabla^2_\gamma F(\hat\gamma^{-z_j}_S(\tau),S) + \frac{\tau}{n}\nabla^2_\gamma h(\hat\gamma^{-z_j}_S(\tau),z_j)\right)^{-1} \right\|_2 \nonumber \\
    & \ \ \ \ \cdot \frac{1}{n} \left(\sum_{i=1}^n \|B(\tau,z_i)\|_2 + \tau \cdot \|B(\tau,z_j)\|_2 + \|A(\tau)\|_2
    + 2 \cdot \left\| \nabla^2_\gamma h(\hat\gamma^{-z_j}_S(\tau),z_j)\cdot \frac{\partial\hat\gamma^{-z_j}_S(\tau)}{\partial\tau} \right\|_2 \right) \label{thm:2nd_norm:equ:upp_bd_mid} \\
    & \leq O\left( \frac{1}{n} \left(\sum_{i=1}^n \|B(\tau,z_i)\|_2 + \tau \cdot \|B(\tau,z_j)\|_2 + \|A(\tau)\|_2
    + 2 \cdot \left\| \nabla^2_\gamma h(\hat\gamma^{-z_j}_S(\tau),z_j)\cdot \frac{\partial\hat\gamma^{-z_j}_S(\tau)}{\partial\tau} \right\|_2 \right) \right).
    \label{thm:2nd_norm:equ:upp_bd}
\end{align}
where eq. (\ref{thm:2nd_norm:equ:upp_bd}) is obtained by inserting eq. (\ref{thm:1st_norm:equ:upp_term_1}) (in the proof of Theorem \ref{thm:1st_norm}) into eq. (\ref{thm:2nd_norm:equ:upp_bd_mid}).
Thus, the remaining is to upper bound the norms of $A(\tau)$, $B(\tau,z)$ and $\nabla^2_\gamma h(\hat\gamma^{-z_j}_S(\tau),z_j) \cdot \frac{\partial\hat\gamma^{-z_j}_S(\tau)}{\partial\tau}$.

We first upper bound $A(\tau)$. Applying Theorem \ref{thm:1st_norm}, we have that
$\left\|\frac{\partial\hat\gamma^{-z_j}_S(\tau)}{\partial\tau}\right\|_2 \leq O\left(\frac{1}{n}\right)$.
Applying Assumptions \ref{ass:cont_diff_scox} and \ref{ass:compact}, we have that $\nabla^2_\gamma \left( \frac{\partial f(\hat\gamma^{-z_j}_S(\tau))}{\partial \gamma_i} \right)$ is bounded on its support $\Gamma$.
As a result, the norm of $\left\| \nabla^2_\gamma \left( \frac{\partial f(\hat\gamma^{-z_j}_S(\tau))}{\partial \gamma_i} \right) \right\|_2$ is also bounded.
Therefore, we have that
\begin{gather}
    \|A(\tau)\|_2
    \leq \sum_{i=1}^K \left\| \nabla^2_\gamma \left( \frac{\partial f(\hat\gamma^{-z_j}_S(\tau))}{\partial \gamma_i} \right) \right\|_2 \cdot \left\| \frac{\partial\hat\gamma^{-z_j}_S(\tau)}{\partial\tau} \right\|_2^2
    \leq \sum_{i=1}^K O(1) \cdot O\left(\frac{1}{n^2}\right)
    = O\left(\frac{1}{n^2}\right).
    \label{thm:2nd_norm:equ:upp_term_A}
\end{gather}

For $B(\tau,z)$, we similarly have that
\begin{gather}
    \|B(\tau,z)\|_2 \leq O\left(\frac{1}{n^2}\right).
    \label{thm:2nd_norm:equ:upp_term_B}
\end{gather}

To upper bound the norm of $\nabla^2_\gamma h(\hat\gamma^{-z_j}_S(\tau),z_j) \cdot \frac{\partial\hat\gamma^{-z_j}_S(\tau)}{\partial\tau}$, we apply Theorem {\ref{thm:1st_norm}}, Assumptions \ref{ass:cont_diff_scox} and \ref{ass:compact}, and have that
\begin{align}
    & \left\| \nabla^2_\gamma h(\hat\gamma^{-z_j}_S(\tau),z_j) \cdot \frac{\partial\hat\gamma^{-z_j}_S(\tau)}{\partial\tau} \right\|_2 \nonumber \\
    &\quad \leq \left\| \nabla^2_\gamma h(\hat\gamma^{-z_j}_S(\tau),z_j) \right\|_2 \cdot \left\| \frac{\partial\hat\gamma^{-z_j}_S(\tau)}{\partial\tau} \right\|_2 \leq O(1) \cdot O\left(\frac{1}{n}\right)
    = O\left(\frac{1}{n}\right).
    \label{thm:2nd_norm:equ:upp_term_3}
\end{align}

Inserting eqs. (\ref{thm:2nd_norm:equ:upp_term_A}), (\ref{thm:2nd_norm:equ:upp_term_B}) and (\ref{thm:2nd_norm:equ:upp_term_3}), into eq. (\ref{thm:2nd_norm:equ:upp_bd}), we eventually have that
\begin{gather*}
    \left\| \frac{\partial^2 \hat\gamma^{-z_j}_S(\tau)}{\partial\tau^2} \right\|_2
    \leq O\left(\frac{1}{n} \left(
        \sum_{i=1}^n O\left(\frac{1}{n^2}\right) + \tau \cdot O\left(\frac{1}{n^2}\right) + O\left(\frac{1}{n^2}\right) + 2 \cdot O\left(\frac{1}{n}\right)
    \right) \right)
    = O\left(\frac{1}{n^2}\right).
\end{gather*}

The proof is completed.
\end{proof}

\section{Proofs in Section \ref{sec:bif_algorithm}}
\label{app:proof_bif_certified}

This section provides the detailed proofs in Section \ref{sec:bif_algorithm}.

\begin{proof}[Proof of Theorem \ref{thm:vi_certified}]
Let $\hat\lambda_S$, $\hat\lambda_{S-\{z_j\}}$ be the variational parameters learned on $S$ and $S-\{z_j\}$, respectively. Let $\hat\lambda^{-z_j}_S = \mathcal{A}_{\mathrm{VI}}(\hat\lambda_S,z_j) = \hat\lambda_S - \mathcal{I}_{\mathrm{VI}}(z_j)$ be the processed variational parameter. Then, to establish a certified knowledge removal guarantee for $\mathcal{A}_{\mathrm{VI}}(\hat\lambda_S,z_j)$, we need to bound the KL divergence $\mathrm{KL}(p_{\hat\lambda^{-z_j}}(\theta) \| p_{\hat\lambda_{S-\{z_j\}}}(\theta))$.

For the brevity, we denote that
\begin{gather*}
    p_{\hat\lambda^{-z_j}_S} = \mathcal{N}(\mu_1, \sigma_1^2 I),
\end{gather*}
where $\hat\lambda^{-z_j}_S=(\mu_{11},\cdots,\mu_{1d},\sigma_{11},\cdots,\sigma_{1d})$, $0 < M_1 \leq \sigma_{11},\cdots,\sigma_{1d} \leq M_2$, and
\begin{gather*}
    p_{\hat\lambda_{S-\{z_j\}}} = \mathcal{N}(\mu_2, \sigma_2^2 I),
\end{gather*}
where $\hat\lambda_{S-\{z_j\}}=(\mu_{21},\cdots,\mu_{2d},\sigma_{21},\cdots,\sigma_{2d})$, $0 < M_1 \leq \sigma_{21},\cdots,\sigma_{2d} \leq M_2$.

We then have the following (we assume that $\Theta=\R^d$),
\begin{align}
    \mathrm{KL}&(p_{\hat\lambda^{-z_j}_S}(\theta)\|p_{\hat\lambda_{S-\{z_j\}}}(\theta)) \nonumber \\
    &= \int_{\Theta} \log\left(\frac{p_{\hat\lambda^{-z_j}_S}(\theta)}{p_{\hat\lambda_{S-\{z_j\}}}(\theta)}\right) p_{\hat\lambda^{-z_j}_S}(\theta) \mathrm{d}\theta \nonumber \\
    &= -\frac{1}{2}\sum_{k=1}^d \int \left(\frac{(\theta_k-\mu_{1k})^2}{\sigma_{1k}^2} - \frac{(\theta_k-{\mu_{2k}})^2}{\sigma_{2k}^2}\right)p_{\hat\lambda^{-z_j}_S}(\theta_k)\mathrm{d}\theta_k - \sum_{k=1}^d \log \frac{\sigma_{1k}}{\sigma_{2k}} \nonumber \\
    &= -\frac{1}{2}\sum_{k=1}^d \left(1 - \frac{\sigma_{1k}^2}{\sigma_{2k}^2} - \frac{(\mu_{1k}-\mu_{2k})^2}{\sigma_{2k}^2} \right) - \sum_{k=1}^d \log \frac{\sigma_{1k}}{\sigma_{2k}} \nonumber \\
    &= \frac{1}{2}\sum_{k=1}^d \left(\frac{\sigma_{1k}^2-\sigma_{2k}^2}{\sigma_{2k}^2} + \frac{(\mu_{1k}-\mu_{2k})^2}{\sigma_{2k}^2} + 2\log\left( 1 + \frac{\sigma_{2k}-\sigma_{1k}}{\sigma_{1k}} \right) \right) \nonumber \\
    &\leq \frac{1}{2}\sum_{k=1}^d \left(\frac{|\sigma_{1k}+\sigma_{2k}|\cdot|\sigma_{1k}-\sigma_{2k}|}{\sigma_{2k}^2} + \frac{(\mu_{1k}-\mu_{2k})^2}{\sigma_{2k}^2} + 2\cdot\frac{|\sigma_{1k}-\sigma_{2k}|}{\sigma_{1k}} \right) \nonumber \\
    &\leq \frac{1}{2}\sum_{k=1}^d \left(\frac{2M_2|\sigma_{1k}-\sigma_{2k}|}{M_1^2} + \frac{(\mu_{1k}-\mu_{2k})^2}{M_1^2} + \frac{2|\sigma_{1k}-\sigma_{2k}|}{M_1} \right) \nonumber \\
    &= \frac{1}{2 M_1^2}\sum_{k=1}^d \left(2(M_1+M_2)|\sigma_{1k}-\sigma_{2k}| + (\mu_{1k}-\mu_{2k})^2 \right) \nonumber \\
    &\leq \frac{1}{2 M_1^2}\left(2(M_1+M_2)\|\hat\lambda^{-z_j}_S - \hat\lambda_{S-\{z_j\}}\|_1 + \|\hat\lambda^{-z_j}_S - \hat\lambda_{S-\{z_j\}}\|_2^2 \right) = \varepsilon_{\hat\lambda_S,z_j}. \label{thm:equ:vi_certified_1}
\end{align}
Therefore, $\mathcal{A}_{\mathrm{VI}}(\hat\lambda_S,z_j)$ performs $\varepsilon_{\hat\lambda_S,z_j}$-certified knowledge removal.

The proof is completed.
\end{proof}


\begin{proof}[Proof of Theorem \ref{thm:mcmc_infty_kl}]
Let the processed model be $p^{-z_j}(\theta) = \mathcal{A}_{\mathrm{MCMC}}(p(\theta|S),z_j)$. For the brevity, we denote that
\begin{gather*}
    p_1(\theta) = p^{-z_j}(\theta) = \mathcal{N}(\theta_1^\prime, (n J(\theta_1))^{-1}),
\end{gather*}
where $\theta_1^\prime = \theta_1-\mathcal{I}_{\mathrm{MCMC}}(z_j)$, and
\begin{gather*}
    p_2(\theta) = p(\theta|S-\{z_j\}) = \mathcal{N}(\theta_2, ((n-1) J(\theta_2))^{-1}).
\end{gather*}
Then, to establish the certified removal guarantee, we need to calculate the KL divergence $\mathrm{KL}(p_1(\theta)\|p_2(\theta))$.


Since
\begin{gather*}
    p_1(\theta)=\frac{1}{\sqrt{(2\pi)^d |nJ(\theta_1)|^{-1}}} \exp\left(-\frac{1}{2} (\theta-\theta_1^\prime)^T (nJ(\theta_1))(\theta-\theta_1^\prime) \right), \\
    p_2(\theta)=\frac{1}{\sqrt{(2\pi)^d |(n-1)J(\theta_2)|^{-1}}} \exp\left(-\frac{1}{2} (\theta-\theta_2)^T ((n-1)J(\theta_2))(\theta-\theta_2) \right),
\end{gather*}
we have that (we assume that $\Theta = \R^d$)
\begin{align*}
    \mathrm{KL}&(p_1(\theta) \| p_2(\theta)) \\
    =& \int_{\Theta} \log \left( \frac{p_1(\theta)}{p_2(\theta)} \right) p_1(\theta) \mathrm{d} \theta \\
    =& -\frac{1}{2} \int_{\Theta}\left( (\theta-\theta_1^\prime)^T (nJ(\theta_1))(\theta-\theta_1^\prime) - (\theta-\theta_2)^T ((n-1)J(\theta_2))(\theta-\theta_2)\right) p_1(\theta) \mathrm{d}\theta \\
    &- \frac{1}{2}\log\frac{|nJ(\theta_1)|^{-1}}{|(n-1)J(\theta_2)|^{-1}} \\
    =& -\frac{1}{2} \left(\mathrm{tr}(nJ(\theta_1)(nJ(\theta_1))^{-1}) - \mathrm{tr}((n-1)J(\theta_2)(nJ(\theta_1))^{-1}) \right. \\
    & \hspace{3em} \left. - (\theta_1^\prime-\theta_2)^T ((n-1)J(\theta_2))(\theta_1^\prime-\theta_2)\right)
    +\frac{1}{2}\left(\frac{n}{n-1}\right)^d \log\frac{|J(\theta_1)|}{|J(\theta_2)|} \\
    =& \frac{1}{2} \left((n-1) (\theta_1^\prime-\theta_2)^T J(\theta_2)(\theta_1^\prime-\theta_2) +\frac{n-1}{n}\mathrm{tr}(J(\theta_2)J^{-1}(\theta_1)) -d 
    +\left(\frac{n}{n-1}\right)^d \log\frac{|J(\theta_1)|}{|J(\theta_2)|} \right) \\
    =& O\left( (n-1) (\theta_1^\prime-\theta_2)^T J(\theta_2)(\theta_1^\prime-\theta_2) +\mathrm{tr}\left(J^{-1}(\theta_1)(J(\theta_2)-J(\theta_1))\right) +\log\frac{|J(\theta_1)|}{|J(\theta_2)|} \right) \\
    =& O(\varepsilon_{\theta_1,z_j}).
\end{align*}
which means that $\mathrm{KL}(p^{-z_j}(\theta)\|p(\theta|S-\{z_j\}))=O(\varepsilon_{\theta_1,z_j})$. Therefore, $\mathcal{A}_{\mathrm{MCMC}}(p(\theta|S),z_j)$ performs $O(\varepsilon_{\theta_1,z_j})$-certified knowledge removal.

The proof is completed.
\end{proof}

\section{Proofs in Section \ref{sec:generalization}}
\label{app:generalization}
This section provides the detailed proofs in Section \ref{sec:generalization}.

We derive generalization bounds for BIF under the PAC-Bayesian framework \citep{mcallester1998some,mcallester1999pac,mcallester2003pac}. The framework can provide guarantees for randomized predictors (e.g. Bayesian predictors).

Specifically, let $Q$ a distribution on the parameter space $\Theta$, $P$ denotes the prior distribution over the parameter space $\Theta$. Then, the expected risks $\mathcal{R}(Q)$ is bounded in terms of the empirical risk $\mathcal{\hat R}(Q,S)$ and KL-divergence $\mathrm{KL}(Q\|P)$ by the following result from PAC-Bayes.

\begin{lemma}[cf. \citet{mcallester2003pac}, Theorem 1]
\label{lem:pac_bayes_gen}
For any real $\delta \in (0,1)$, with probability at least $1-\delta$, we have the following inequality for all distributions $Q$:
\begin{gather}
    \mathcal{R}(Q) \leq \mathcal{\hat R}(Q,S) + \sqrt{\frac{\mathrm{KL}(Q\|P)+\log \frac{1}{\delta} + \log n + 2}{2n-1}}. \label{lem:equ:pac_bayes_gen}
\end{gather}
\end{lemma}

Based on Lemma \ref{lem:pac_bayes_gen}, we prove the generalization bounds in Theorem \ref{thm:gen_if} and Theorem \ref{thm:gen_mcmc}.

\begin{proof}[proof of Theorem \ref{thm:gen_if}]
Let the prior distribution $P$ be the standard Gaussian distribution $\mathcal{N}(0,I)$, $Q$ be the distribution that obtained after conducting variational inference BIF. Suppose the density functions of $P$, $Q$ are $p(\theta)$, $q(\theta)$ respectively, then we have
\begin{gather*}
    p(\theta)=\frac{1}{\sqrt{(2\pi)^d}} \exp\left(-\frac{\|\theta\|^2}{2}\right), \\
    q(\theta)=\frac{1}{\sqrt{(2\pi)^d} \prod_{k=1}^d \sigma_k^\prime} \exp\left(-\sum_{k=1}^d  \frac{(\theta_k-{\mu_k^\prime})^2}{2{\sigma_k^\prime}^2} \right),
\end{gather*}
where $\mu_k^\prime = \mu_k-\Delta_{\mu_k}$ and $\sigma_k^\prime = \sigma_k-\Delta_{\sigma_k}$.

Therefore, we can calculate the KL-divergence $\mathrm{KL}(Q\|P)$ as follows (where we assume that $\Theta=\R^d$),
\begin{align}
    \mathrm{KL}&(Q\|P) \nonumber \\
    &=\int_\Theta \log\left(\frac{q(\theta)}{p(\theta)}\right) q(\theta) \mathrm{d}\theta \nonumber \\
    &= -\frac{1}{2}\sum_{k=1}^d\int_{\Theta}\left(\frac{(\theta_k-\mu_k^\prime)^2}{{\sigma_k^\prime}^2} - \theta_k^2 \right) q(\theta)\mathrm{d}\theta -\sum_{k=1}^d\log{\sigma_k^\prime} \nonumber \\
    &= -\frac{1}{2}\sum_{k=1}^d\left(1-{\mu_k^\prime}^2-{\sigma_k^\prime}^2\right) - \sum_{k=1}^d\log{\sigma_k^\prime} \nonumber \\
    &= \frac{1}{2}\left( \|\lambda - \Delta_\lambda\|^2 - d - 2\sum_{k=1}^d\log{\sigma_k^\prime} \right) \nonumber \\
    &\leq \frac{1}{2}\left( \|\Delta_\lambda\|^2 + 2 \|\lambda\| \cdot \|\Delta_\lambda\| + \|\lambda\|^2 - 2 \sum_{k=1}^d\log\left(\sigma_k-\Delta_{\sigma_k}\right) - d \right) \nonumber \\
    &= \frac{1}{2} \left( C_{\lambda, \Delta_\lambda} - d \right),
    \label{thm:equ:vi_kl}
\end{align}
where
\begin{gather*}
C_{\lambda,\Delta_\lambda}=\|\Delta_\lambda\|^2 + 2 \|\lambda\| \cdot \|\Delta_\lambda\| + \|\lambda\|^2 - 2 \sum_{k=1}^d\log\left(\sigma_k-\Delta_{\sigma_k}\right).
\end{gather*}
Eq. (\ref{thm:equ:vi_kl}) gives an upper bound of the KL-divergence between the obtained distribution after BIF and the prior distribution. Inserting eq. (\ref{thm:equ:vi_kl}) into eq. (\ref{lem:equ:pac_bayes_gen}) in Lemma \ref{lem:pac_bayes_gen}, we obtain the PAC-Bayesian generalization bound.

Furthermore, since all the conditions in Theorem \ref{thm:vi_bif} hold, we can apply Theorem \ref{thm:1st_norm} and have that $|\Delta_{\sigma_k}| \leq \|\Delta_{\lambda}\| \leq O(1/n)$. Besides, by applying the conditions in Theorem \ref{thm:vi_certified}, we have that for any $1\leq k \leq d$, $\sigma_k \geq M_1 > 0$, where $M_1$ is a constant. Therefore, we have that
\begin{align*}
C_{\lambda,\Delta_\lambda}
\leq& \|\Delta_\lambda\|^2 + 2 \|\lambda\| \cdot \|\Delta_\lambda\| + \|\lambda\|^2 - 2d\log\left(M_1 - \|\Delta_\lambda\| \right) \\
=& O\left( \log\left(M_1 - \|\Delta_\lambda\|\right) \right) \\
\leq& O(1).
\end{align*}

The proof is completed.
\end{proof}

\begin{proof}[proof of Theorem \ref{thm:gen_mcmc}]
Let the prior distribution $P$ be the standard Gaussian distribution $\mathcal{N}(0,I)$, $Q$ be the processed distribution that obtained after conducting MCMC BIF. Suppose the density functions of $P$, $Q$ are $p(\theta)$, $q(\theta)$ respectively, then we have
\begin{gather*}
    p(\theta)=\frac{1}{\sqrt{(2\pi)^d}} \exp\left(-\frac{\|\theta\|^2}{2}\right), \\
    q(\theta)=\frac{1}{\sqrt{(2\pi)^d |nJ(\theta_1)|^{-1}}} \exp\left(-\frac{1}{2} (\theta-\theta_1^\prime)^T (nJ(\theta_1))(\theta-\theta_1^\prime) \right),
\end{gather*}
where $\theta_1^\prime = \theta_1 - \Delta_{\theta_1}$.

Thus, we can calculate the KL-divergence $\mathrm{KL}(Q\|P)$ as follows (where we assume that $\Theta=\R^d$),
\begin{align}
    \mathrm{KL}&(Q\|P) \nonumber \\
    &=\int_\Theta \log\left(\frac{q(\theta)}{p(\theta)}\right) q(\theta) \mathrm{d}\theta \nonumber \\
    &= -\frac{1}{2}\left(\int_{\Theta}\left((\theta-\theta_1^\prime)^T(nJ(\theta_1))(\theta-\theta_1^\prime)-\|\theta\|^2\right)q(\theta)\mathrm{d}\theta -\log\left|nJ(\theta_1)\right| \right) \nonumber \\
    &= -\frac{1}{2}\left(\mathrm{tr}(I) - \mathrm{tr}\left((nJ(\theta_1))^{-1}\right) - \|\theta_1^\prime\|^2 - \log \left|nJ(\theta_1)\right| \right) \nonumber \\
    &= \frac{1}{2}\left(\|\theta_1^\prime\|^2 + \frac{1}{n}\mathrm{tr}(J^{-1}(\theta_1)) - d + \log\left(n^d |J(\theta_1)| \right) \right) \nonumber \\
    &\leq \frac{1}{2}\left(\|\Delta_{\theta_1}\|^2 + 2\|\theta_1\| \cdot \|\Delta_{\theta_1}\| +  \|\theta_1\|^2 + \frac{1}{n}\mathrm{tr}(J^{-1}(\theta_1)) + \log|J(\theta_1)| + d(\log n - 1)  \right) \nonumber \\
    &= \frac{1}{2} \left( C_{p,\Delta_{\theta_1}} + d(\log n - 1) \right),
    \label{thm:equ:mcmc_kl}
\end{align}
where
\begin{gather*}
C_{p,\Delta_{\theta_1}}=\|\Delta_{\theta_1}\|^2 + 2\|\theta_1\| \cdot \|\Delta_{\theta_1}\| + \|\theta_1\|^2 + \frac{1}{n}\mathrm{tr}(J^{-1}(\theta_1)) + \log|J(\theta_1)|.
\end{gather*}
Eq. (\ref{thm:equ:mcmc_kl}) bounds the KL-divergence between the processed distribution after conducting BIF and the prior distribution. Inserting eq. (\ref{thm:equ:mcmc_kl}) into eq. (\ref{lem:equ:pac_bayes_gen}) in Lemma \ref{lem:pac_bayes_gen}, we then obtain the PAC-Bayesian generalization bound.

Eventually, since all the conditions in Theorem \ref{thm:mcmc_if} hold, we can apply Theorem \ref{thm:1st_norm} and have that $\|\Delta_{\theta_1}\| \leq O(1/n)$. Therefore, we have the following,
\begin{gather*}
C_{p,\Delta_{\theta_1}}= O\left( \|\theta_1\|^2 + \frac{1}{n}\mathrm{tr}(J^{-1}(\theta_1)) + \log|J(\theta_1)| \right) = O(1).
\end{gather*}

The proof is completed.
\end{proof}

\end{document}